%% file: ICLR_2026.tex
\PassOptionsToPackage{dvipsnames}{xcolor}

\documentclass{article}

\usepackage{natbib}
\usepackage{times}
\usepackage{iclr2026_conference}      % To produce the REVIEW version
\iclrfinalcopy
\usepackage{fancyhdr}
\fancypagestyle{preprint}{
  \fancyhf{} % clear all fields
  \fancyhead[L]{\footnotesize\textit{Preprint.}}
  \fancyfoot[C]{\thepage}            % keep page number centered (optional)
}
\usepackage[utf8]{inputenc} % allow utf-8 input
\usepackage[T1]{fontenc}    % use 8-bit T1 fonts
\usepackage{hyperref}       % hyperlinks
\usepackage{url}            % simple URL typesetting
\usepackage{booktabs}       % professional-quality tables
\usepackage{amsfonts}       % blackboard math symbols
\usepackage{nicefrac}       % compact symbols for 1/2, etc.
\usepackage{microtype}      % microtypography
\usepackage{xcolor}         % colors

\input{preamble}

\usepackage{bm,mathrsfs,graphics,float,amssymb,subeqnarray,setspace,graphicx,amsthm,epstopdf,color}
\usepackage{enumitem}
\usepackage{soul}
\usepackage{wrapfig}
\usepackage{multirow,multicol,makecell}
\usepackage{wrapfig}
\usepackage{physics}
\usepackage{amsmath}
\usepackage{adjustbox}
\usepackage{changepage}

\definecolor{linkcolor}{HTML}{ED1C24}
\definecolor{my_blue}{HTML}{00B0F0}
\definecolor{my_purple}{HTML}{7030A0}
\definecolor{my_light}{HTML}{CAB2D6}
\definecolor{citecolor}{HTML}{0071BC}

\definecolor{tableau_blue}{HTML}{1F77B4}
\definecolor{tableau_orange}{HTML}{FF7F0E}
\definecolor{tableau_green}{HTML}{2CA02C}
\definecolor{tableau_red}{HTML}{D62728}
\newcommand{\myred}[1]{\textbf{\color{tableau_red}\textbf{{#1}}}}
\newcommand{\green}[1]{\textbf{\color{tableau_green}\textbf{{#1}}}}
\newcommand{\blue}[1]{\textbf{\color{tableau_blue}\textbf{{#1}}}}
\newcommand{\orange}[1]{\textbf{\color{tableau_orange}\textbf{{#1}}}}

\newtheorem{lemma}{Lemma}
\newtheorem{theorem}{Theorem}
\newtheorem{theorem*}{Finding}

\newtheorem{proposition}{Proposition}
\newtheorem{definition}{Definition}

\newtheorem{condition}{Condition}

\newcommand{\Sb}{\boldsymbol{\Sigma}_B}
\newcommand{\Sw}{\boldsymbol{\Sigma}_W}

\newcommand{\NCRLTerm}{Action Collapse}
\newcommand{\ncrlTerm}{Action Collapse}
\newcommand{\headterm}{action selection layer}
\newcommand{\headact}{state-action}
\newcommand{\ours}{ACPG}

\usepackage{caption}
\title{Imitate Optimal Policy: Prevail and Induce \\Action Collapse in Policy Gradient}

\author{
\textbf{Zhongzhu Zhou}$^{1,3}$, 
\textbf{Yibo Yang}$^{2}$, 
\textbf{Ziyan Chen}$^{1}$, 
\textbf{Fengxiang Bie}$^{1}$, 
\textbf{Haojun Xia}$^{1}$\\
\textbf{Xiaoxia Wu}$^{3}$, 
\textbf{Robert Wu}$^{3}$, 
\textbf{Ben Athiwaratkun}$^{3}$, 
\textbf{Bernard Ghanem}$^{2}$, 
\textbf{Shuaiwen Leon Song}$^{1,3}$\\
$^{1}$ University of Sydney,
$^{2}$ King Abdullah University of Science and Technology,
$^{3}$ Together AI
}

\begin{document}

\maketitle
\input{sec/0_abstract}    
\input{sec/1_intro}
\input{sec/2_pre}
\input{sec/3_emp}
\input{sec/4_main}
\input{sec/5_exp}
\input{sec/6_rel}
\input{sec/7_con}

\newpage
{
    \small
    \bibliographystyle{iclr2026_conference}
    \bibliography{main}
}

\newpage
\appendix
\input{sec/9_app}

% WARNING: do not forget to delete the supplementary pages from your submission 
% \input{sec/X_suppl}

\end{document}

%% file: preamble.tex
\usepackage{xcolor}

%% file: sec/0_abstract.tex
\begin{abstract}
Policy gradient (PG) methods in reinforcement learning frequently utilize deep neural networks (DNNs) to learn a shared backbone of feature representations used to compute likelihoods in an \headterm. 
Numerous studies have been conducted on the convergence and global optima of policy networks, but few have analyzed representational structures of those underlying networks.
While training an optimal policy DNN, we observed that under certain constraints, a gentle structure resembling neural collapse -- which we refer to as \textit{\ncrlTerm} (AC) -- emerges.
This suggests that
1) the \headact\ activations (i.e. last-layer features) sharing the same optimal actions collapse towards those optimal actions' respective mean activations;
2) the variability of activations sharing the same optimal actions converges to zero;
3) the weights of \headterm\ and the mean activations collapse to a simplex equiangular tight frame (ETF).
Our early work showed those aforementioned constraints to be necessary for these observations.
Since the collapsed ETF of optimal policy DNNs maximally separates the pair-wise angles of all actions in the \headact\ space, we naturally raise a question:
\textit{can we learn an optimal policy using an ETF structure as a (fixed) target configuration in the \headterm?}
Our analytical proof shows that learning activations with a fixed ETF as \headterm\
naturally leads to the \NCRLTerm.
We thus propose the \textit{\NCRLTerm\ Policy Gradient} (\ours) method, which accordingly affixes a synthetic ETF as our \headterm.
\ours\ induces the policy DNN to produce such an ideal configuration in the \headterm\ while remaining optimal.
Our experiments across various OpenAI Gym environments demonstrate that our technique can be integrated into any discrete PG methods and lead to favorable reward improvements more quickly and robustly.
\end{abstract}

%% file: sec/1_intro.tex
\section{Introduction} \label{sec:INTRO}

Reinforcement learning (RL) involves training agents to interact with environments by selecting actions aiming to maximize expected cumulative rewards or ``returns''. Classical methods such as Q-learning~\cite{sutton1984on} and SARSA~\cite{watkins1989learning}
do this by learning a accompanying value function.
In contrast, ``policy gradient'' (PG) algorithms ~\cite{sutton1999policy} directly optimize the action selection layer~\footnote{The last layer of PG DNNs that computes the likelihoods for each action, a.k.a. the action classifier/head.} to maximize return by modeling action probabilities from state features. Modern PG methods use deep neural networks (DNNs) to estimate these probabilities.

Recent breakthroughs in artificial intelligence---such as GPT-4~\cite{peng2023instruction}---have been driven by large language models (LLMs), with much of their later success attributed to RL. Reinforcement learning from human feedback (RLHF)~\cite{RLHF, sriyash2024, dai2024safe} introduced a new paradigm of LLM post-training, many examples of which incorporate PG, such as TRPO~\cite{silver2014deterministic,schulman2015trust}, PPO~\cite{schulman2017proximal}, and DPO~\cite{rafailov2023direct}. %, though the standard PG algorithm remains a strong and widely used approach.

%Reinforcement learning (RL) is particular such AI arguments performance with potential exposure to continuous data streams.

% \begin{figure}[!tbp]
% 	\begin{center}
% 		\includegraphics[width=0.49\textwidth]{fig/fig1.png}
% 		\caption{Illustration of equiangular separation (a) and non-equiangular separation (b) in a 3D space. \textit{'\NCRLTerm'}\ reveals the structure in (a), where layer-layer state feature are collapsed into their corresponded-optimal-action's state features' means with \textit{\textbf{maximal equiangular separation}} as a simplex ETF, and action determining layer are aligned with the same structure. In real RL environment, the environment has not been fully explored the policy neural networks do not satisfy such a structure, as illustrated in (b) for an example. As some states are minorly explored, their vectors lie in a close position, the action separability of the policy neural network for these states degrades.}
% 		\label{fig:ILL}
% 	\end{center}
% \end{figure}

Despite these empirical successes in RL, a substantial gap persists between the theory and practice of deep RL. Theoretical research in RL has primarily focused on the convergence of RL algorithms~\cite{gaur2023global, fan2020theoretical, asadi2023td, qlearning_analysis} and understanding how to effectively train an optimal policy~\cite{xiong2022deterministic}~\footnote{The optimal policy is defined as a policy which outputs the optimal action and achieves the highest return.}. Other studies attempt to tackle the exploration–exploitation dilemma which requires policy networks to strike a balance between exploring previously unvisited states in the experience replay buffer~\cite{lin1992self} and exploiting existing policy efficiently.
%to train policy model \textit{better}. At the same time, because the accumulated reward the model acquired in the environment contains huge variance, the policy shall be fully \textit{exploited} in interaction with the environment to generate experience datasets. 
Recent works have motivated policies, either explicitly or implicitly~\cite{thrun1992efficient}, to take exploratory actions in order to discover new outcomes, even in the absence of immediate rewards or theoretical guarantees of convergence~\cite{deterministic, strehl2005theoretical}. However,
it is still often unclear how an optimal policy is to be derived, and what underlying DNN structures would achieve such optimality.
%Nevertheless, the impact of the \textit{'Exploration \& Exploitation'} dilemma on the final trained policy~\cite{sutton1999policy} and how the optimal policy is achieved within the \textit{'Exploration vs Exploitation'} is unknown.  % In this paper, we constrained the scope, assuming time is unlimited, and the environment is fully explored. The final trained optimal policy is investigated empirically under these ideal conditions.   %Inspired by recent research work, neural collapse~\cite{papyan2020prevalence} in classification tasks, , to our best knowledge,  none investigate the structure of optimal policy DNNs, after }.  %Thus, balancing between  \textit{'Exploration'} vs \textit{'Exploitation'} is an eternal dilemma and valuable for theoretical analysis. But the overall target of all these theoretical analysis is still to find better / optimal policy across training.

A set of phenomena known as ``Neural Collapse'' (NC) ~\cite{papyan2020prevalence}, characterizes the structure of optimal classification DNNs trained to convergence using canonical objectives such as cross-entropy (CE)~\cite{papyan2020prevalence} or mean square error (MSE)~\cite{han2021neural}.
Several works have demonstrated that 1) the last-layer features of data samples to collapse to their respective class means; 2) those means construct a simplex equiangular tight frame (ETF)~\cite{ETF}; 3) classifier weights also converge to this ETF; and 4) classification decisions converge to a near class centre classifier.
Beyond its elegant geometric symmetry, the theoretical breakthrough of NC has proven valuable in inspiring new training methods for classification models, particularly under conditions of data imbalance and scarcity \cite{xie2023neural,yang2023neural}.

PG methods share key similarities with classification tasks. For instance, PG is applied to the classification of positive and unlabeled data, where a policy DNN is trained to infer data labels ~\cite{li2019learning}. A notable example is Goal-Conditioned Supervised Learning (GCSL), which reformulated RL states and target rewards as input and the corresponding (optimal) action as label ~\cite{ghosh2021learning}. % Both methods generate training data from observed trajectories and apply classification techniques. 
Based on the existence of NC and conceptual parallels between RL and classification, we aim to investigate the structures of optimal DNNs for PG. From the perspective of cost function formulation, PG is essentially CE weighted by a Q-value function~\cite{sutton1999policy}. Therefore, we empirically assess whether phenomena analogous to NC emerge during PG training in several RL environments.

A similar phenomenon---which we refer to as \textit{\NCRLTerm}---is observed in our ideal RL environments under certain constraints. However, real-world RL environments train DNNs within a limited time frame~\cite{zhang2019convergence}, resulting in data scarcity. RL samples \headact\ sets from the environment~\cite{sutton1999policy}, which may exhibit significant imbalance across different states due possibly to greedy exploitation and incomplete knowledge of the environment, and usually to the natural imbalance in the task at hand. Since \NCRLTerm\ is rarely observed in real-world problems, we further consider: \textbf{can we learn an optimal policy using an ETF structure as a (fixed) target configuration in the \headterm\ for realistic RL problems where ideal condition cannot be assumed?}

In this paper, we propose \NCRLTerm\ Policy Gradient (ACPG), which fixes the \headterm\ as a synthetic simplex ETF.
Our experiments demonstrate that ACPG achieves higher returns than previous PG methods, with faster and more robust convergence. In addition, our theoretical analysis shows that even in the absence of idealized empirical constraints, the activation means still collapse to a ETF under ACPG training. In other words, \NCRLTerm\ naturally develops in practical RL environments using our proposed technique. The main contributions of this work can be summarized as follows:

\begin{itemize}
\item To the best of our knowledge, this is the first study to empirically analyze the structure of optimal policy DNNs in fully explored RL environments, in which we observed the \NCRLTerm\ phenomenon (under certain constraints). We note that \NCRLTerm\ may not naturally occur in more realistic environments, due to factors such as insufficient exploration, imbalanced sampling, or the natural imbalance of states/actions.
\item To address this, we propose \ours, which initializes \headterm\ as a fixed simplex ETF structure. Accordingly, we analytically prove that \NCRLTerm\ optimality can still be induced in such realistic environments under \ours.
\item We validate \ours\ across \textbf{10+} Gym environments, \textbf{3+} different PG methods. Our proposed method can be integrated with any discrete PG algorithms, and consistently achieves more stable, faster, and robust convergence, leading to significant performance improvements. %\textbf{2.4X} reward accumulation benefits at most. Our methods can be integrated into any other existing discrete policy gradient methods and bring remarkable improvement.
\end{itemize}

%% file: sec/2_pre.tex
\section{Preliminaries} \label{sec:PRE}
\subsection{Policy Gradient} \label{sec:PG}

The goal of PG is to find an optimal policy, which in our case is determined by a trained DNN parameterized by $\theta$. $\pi_\theta(a \vert s)$ represents the probability of choosing action $a$ regarding environment observation $s$ under policy $\pi_\theta$~\footnote{The subscript for parameter $\theta$ is omitted for the policy $\pi_\theta$ when the policy is present in the superscript of other functions; for example, $d^{\pi}$ and $\Psi^\pi$ should be $d^{\pi_\theta}$ and $\Psi^{\pi_\theta}$ if written in full.}. The value from the reward function depends on the policy, and various algorithms~\cite{schulman2015trust, schulman2017proximal, cobbe2020phasic} can be applied to optimize $\theta$ to maximize the expected return value.
\begin{definition}[Reward/Objective Function of Policy Gradient Methods]
Normally an discrete RL environment contains a state set ${\mathcal{S}}$ and the discrete action space ${A} \text{, with } |A|=K$. We define the set as $\mathcal{S}_k = \{s:opt(s)=a_k\}$, where $opt(s)$ denotes the true optimal action for state $s$.

The reward function is defined as
\begin{equation}\label{eq:RLO}
J(\theta) 
= \sum_{s \in \mathcal{S}} d^\pi(s) \sum_{a \in \mathcal{A}} \pi_\theta(a \vert s) \Psi^\pi,
\end{equation}
where $d^\pi(s)$ is the stationary distribution of the Markov chain induced by $\pi_\theta$ (\textit{i.e.}, the on-policy state distribution under $\pi_\theta$), with $d^\pi(s) = \lim_{t \to \infty} P(S_t = s \vert S_0, \pi)$. $\Psi$ denotes the ``expected return'' when action $a$ is chosen at state $s$. ~\cite{sutton1999policy}
\end{definition}

In RL settings, $\pi_\theta(a \vert s)$ represents a stable probability distribution over actions given a state $s$ when following policy $\pi_\theta$.  The ``expected return'' $\Psi$ quantifies the future reward obtained by taking action $a$ in state $s$, which may vary across different algorithms. It may represent either 1) the remaining reward accumulation along a trajectory: $\sum_{t=\text{after } s,a}^{\infty} r_t$; or 2) the return following action $a_t$: $\sum_{t'=t}^{\infty} r_{t'}$, etc. For simplicity, we assume that a ground-truth value (optimal return) exists. %and is decided by the environment.

According to the PG loss function while optimizing $J(\theta)$, after the environment is fully explored and experiences for all $\mathcal{S}$ and $\mathcal{A}$ are collected, $J(\theta)(\mathbf{H})$ can be expressed as:
\begin{equation} \label{eq:PGLOSS}
J(\theta)(\mathbf{H})
= \sum_{s \in \mathcal{S}}d^{\pi}(s) \cdot \log\frac{e^{h(s;\theta)^T w_{a^*} } }{\sum_{a \in \mathcal{A}} e^{h(s;\theta)^T w_a}} \Psi^\pi,
\end{equation}
where $\mathbf{H}=\{h(s;\theta) : s \in \mathcal S \}$ and $h(s;\theta)\in\mathbb{R}^d$ is the activation, % $K$ is the size of discrete action space, 
$\mathbf{W}=[w_1;\cdots;w_K] \in\mathbb{R}^{K\times d}$ is the \headterm\, and $a^*$ is the optimal action predicted by the model. 

\subsection{Neural Collapse} \label{sec:ETF}

``Neural Collapse'' (NC) was observed and formalized in classification tasks~\cite{papyan2020prevalence}, where feature maps converge to their within-class means, which together with the classifier vectors, collapse to the vertices of a simplex ETF during the terminal phase of training (TPT). %on a balanced dataset. 
See Appendix \ref{app:NC} for details.

\begin{definition}[Simplex Equiangular Tight Frame] \label{eq:ETF}
	A collection of vectors $m_i\in\mathbb{R}^d$, $i=1,2,\cdots,K$, $d\ge K-1$, are said to be a simplex ETF if:
	\begin{equation}\label{eq:ETFM}
		\mathbf{M}=\sqrt{\frac{K}{K-1}}\mathbf{U}\left(\mathbf{I}_K-\frac{1}{K}\mathbf{1}_K\mathbf{1}_K^T\right),
	\end{equation}
	where $\mathbf{M} = [m_1, \cdots, m_K] \in \mathbb{R}^{d \times K}$, $\mathbf{U} \in \mathbb{R}^{d \times K}$ is an orthogonal matrix satisfying $\mathbf{U}^T \mathbf{U} = \mathbf{I}_K$, $\mathbf{I}_K$ is the identity matrix, and $\mathbf{1}_K$ is the all-ones vector.
\end{definition}

All vectors in a simplex ETF have unit $\ell_2$ norm and the same pair-wise angle,
\begin{equation} \label{eq:MIMJ}
	m_i^Tm_j=\frac{K}{K-1}\delta_{i,j}-\frac{1}{K-1}, \forall i, j\in[1,K],
\end{equation}
where $\delta_{i,j}$ equals to $1$ when $i=j$ and $0$ otherwise. The pair-wise cosine similarity $\frac{-1}{K-1}$ is the maximal equiangular separation of $K$ vectors in $\mathbb{R}^d$ \cite{papyan2020prevalence}. 

% Then the neural collapse (NC) phenomenon can be formally described as:

% \textbf{(NC1)} Collapse of within-class variability: $\Sigma_W\rightarrow\mathbf{0}$, and $\Sigma_W:=\mathrm{Avg}_{i,k}\{(h_{k,i}-h_k)(h_{k,i}-h_k)^T\}$, where $h_{k,i}$ is the last-layer feature of the $i$-th sample in the $k$-th class, and $h_k=\mathrm{Avg}_{i}\{h_{k,i}\}$ is the within-class mean of the last-layer features in the $k$-th class;% \emph{i.e.,} $\h_k=\mathrm{Avg}_{i}\{\h_{k,i}\}$;

% \textbf{(NC2)} Convergence to a simplex ETF: $\tilde{h}_k = (h_k-h_G)/||h_k-h_G||, k\in[1,K]$, satisfies Eq. (\ref{eq:MIMJ}), where  $h_G$ is the global mean of the last-layer features, \emph{i.e.,} $h_G=\mathrm{Avg}_{i,k}\{h_{k,i}\}$;

% \textbf{(NC3)} Self duality: $\tilde{h}_k=w_k/||w_k||$, where $w_k$ is the classifier vector of the $k$-th class;

% \textbf{(NC4)} Simplification to the nearest class center prediction: $argmax_k\langle h, w_k\rangle=argmin_k||h-h_k||$, where $h$ is the last-layer feature of a sample to predict for classification.

%% file: sec/3_emp.tex
\section{Empirical Observation} \label{sec:EMP}
In this subsection, we present the main empirical evidence to show \NCRLTerm.

% \textbf{Optimization Methodology} Following common practice, we minimize the policy gradient loss using stochastic gradient descent (SGD) with momentum $0.9$. The weight decay is set to $1{\times}10^{-4}$. The size of experience replay is $128$, in \textbf{T4}~\cite{sund2020bat} GPU, and the other datasets are trained on a single GPU with a batch size of $128$. We train for 100 epoch. The initial learning is annealed by a factor of $10$ at $1/2$ and $3/4$. We sweep over $10$ logarithmically-spaced learning rates between $0.01$ and $0.25$, and $25$ learning rates for the remaining environment between $0.0001$ and $0.25$--picking the model resulting in the best reward return averagely.

First, following Cliff Walking~\cite{demin2009cliff}, we define a similar environment: the agent navigates a \textbf{2x4} grid, where two exits are consistently placed in the upward and downward directions of the central \textbf{2x2} region of the grid. The diagram is shown in Fig.\ref{fig:EMP}.(a). We apply the REINFORCE algorithm~\cite{sutton1999policy} and estimate the returns via Monte-Carlo estimation~\cite{sutton1999policy} using episode samples to update the policy parameter $\theta$. % The agent DNNs are formed by \textit{L} layers with \textit{$D\times D$} hidden layers, (L = 3, 4, 9, 16), (D = 16, 32, 64, 128). 
We constrain the Cliff Walking environment using the following three conditions:

\begin{condition}
The state is fully explored and discrete, i.e., $max (|\mathcal{S}|)=12$ in our cliff walking. 
\end{condition}

\begin{condition}
The cardinality of each state set ${\mathcal{S}}_k$, as defined in DEF.1, remains consistent across different actions $a_k$ in experience replay, such that $|\mathcal{S}_k| = |\mathcal{S}_{k^\prime}|\ \forall k, k^\prime \in K$. Specifically, for the four directional actions in the experience replay of our ideal cliff walking, $\mathcal{S}_1=\mathcal{S}_2=\mathcal{S}_3=\mathcal{S}_4$.
\end{condition}

\begin{condition}
The sampling of states across different actions selected by any policy is expected to be with equal probability when stationary,
%The sampling of each action of each state for any policy is expected to be sampled with equal probability in the final trajectory, 
$\forall s, s^\prime \in \mathcal{S}\ ,\ d^\pi(s) = d^\pi(s^\prime)$.
\end{condition}

In addition to our Ideal Cliff Walking environment, the metrics of \NCRLTerm\ in three other Gym environments, Discrete-Car-Racing (Car-Racing)~\footnote{The reward of Car-Racing is normalized to 0-20 for visualization.}, Breakout-V5 (Breakout) and ALE/Pong-V5 (Pong)~\cite{brockman2016openai} are also investigated.

\textbf{Moment of policy DNNs convergence.} \label{sec:MOM}
During training, we periodically snapshot the parameters of the policy DNN at designated epochs. Given all unique states observed from an environment, we extract both the activations of the \headterm\ and the corresponding weight vectors. 

Let $h(s;\theta)$ --- henceforth shortened as $h$ --- denote the \headact\ activation. We compute the global activation mean $h_G \in \mathbb{R}^d$:
%\begin{equation*}
    $h_G \triangleq Avg_{s \in \mathcal{S}} \{h_{s} \}$, 
%\end{equation*}
where $s$ and $\mathcal{S}$ are defined in Sec.\ref{sec:PRE}. For each action $a_k$, $k \in 1,\dots,K$, we define the mean activations of distinct states that share $a_k$ as the optimal action:
%\begin{equation*}
    $h_{k} \triangleq Avg_{i=1}^{n_k} \{ h_{i, k} \}, \quad k=1,\dots,K$,
%\end{equation*}
where $n_k$ is the number of distinct \headact\ activations that share $a_k$ as the  optimal action.
Similarly, $w_{k}$ is weight of \headterm\ of the $k^{th}$ action. To measure angularity, given two distinct action indexes $k,{k^\prime}$, we denote 
$\cos_h(k,k^\prime) = \left\langle h_{k} - h_G, h_{{k^\prime}} - h_G \right\rangle / (||h_{k} - h_G||_2||h_{k^\prime} - h_G||_2)$ and 
$\cos_w(k, {k^\prime})= \left\langle w_{k}, w_{{k^\prime}}\right\rangle /(||w_{k}||_2||w_{{k^\prime}}||_2)$.

\begin{figure*}[htb]
	\begin{center}		\includegraphics[width=0.99\textwidth]{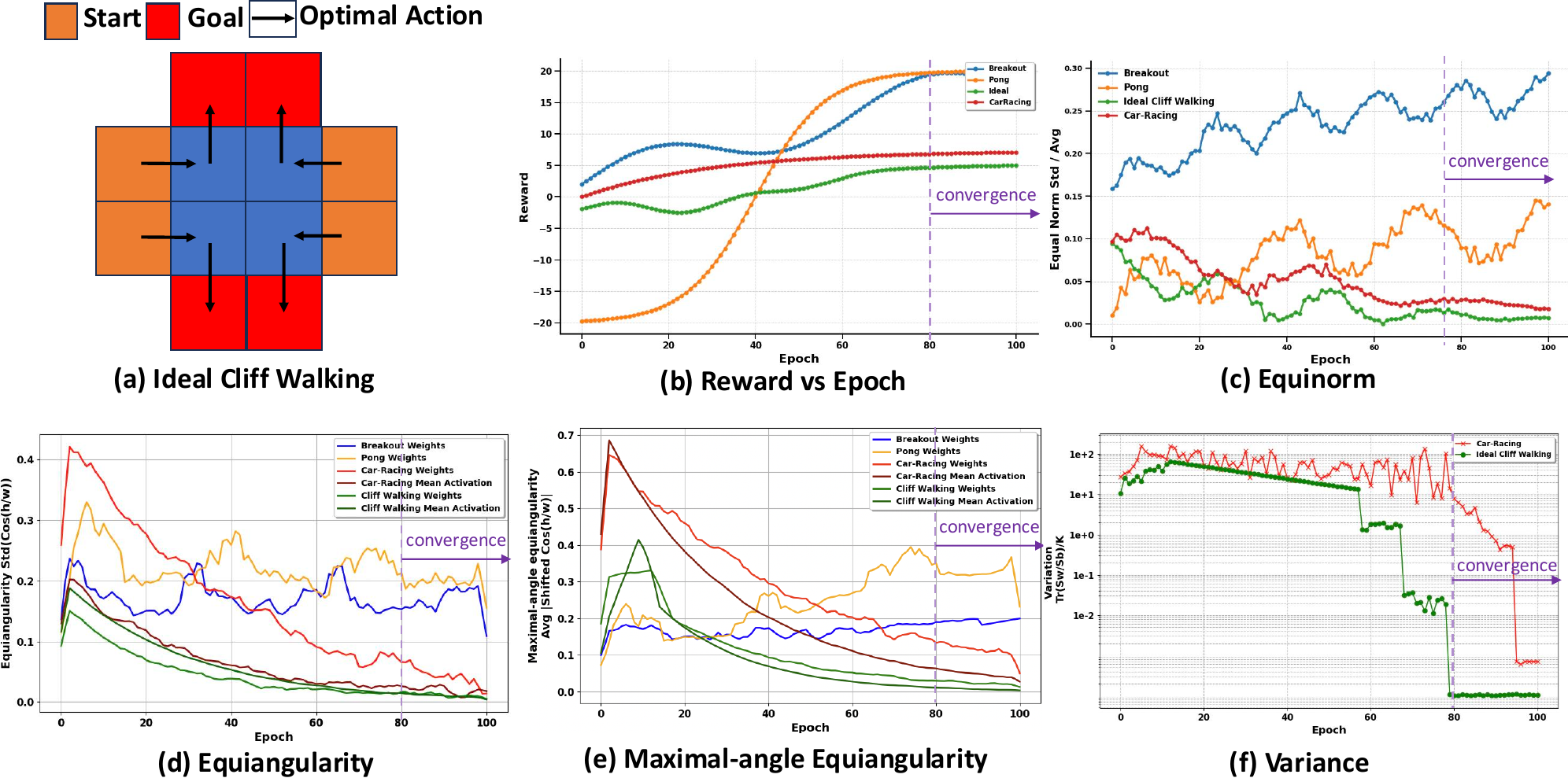}
            \caption{(a) Ideal environment diagram.  (b) Models converge to stability; $y$-axis shows reward accumulation per epoch. (c) Weights of \headterm\ become equinorm; $y$-axis shows $\text{Std}_{k}(\|w_{k}\|_2)/\text{Avg}_k(\|w_{k}\|_2)$. (d) Mean \headact\ activations and weights of \headterm\ approach equiangularity; $y$-axis shows $\text{Std}_{k,{k'}\neq k}(\cos_{h/w}(k, {k'}))$. (e) Mean \headact\ activations and weights of \headterm\ approach maximal-angle equiangularity; $y$-axis shows $\text{Avg}_{k,{k'}}|\cos_{h/w}(k, {k'})+ 1/(K-1)|$. (f) Variation of \headact\ activations sharing the same optimal action; $y$-axis shows $\Tr{\Sw\Sb^{\dagger}}/K$, where $\Tr{\cdot}$ is the trace operator, $\Sw$ is the covariance of $h_{k}$, $\Sb$ is the corresponding covariance between $h_{k}$, and $[\cdot]^{\dagger}$ is Moore-Penrose pseudo-inverse. In \green{Ideal Cliff Walking} and \myred{Car-Racing}, \NCRLTerm\ happens in each figures. However, realistic RL environments such as \blue{Breakout} and \orange{Pong} show ever-decreasing equinormness and equiangularity which shows that they fall into biased optimal and do not form \NCRLTerm.}
            
        \label{fig:EMP}
        \end{center}
        \vspace{-20px}
\end{figure*}

\textbf{Results and Discussion.} Our main findings and results are presented in Fig.\ref{fig:EMP}. \NCRLTerm\ is clearly observed in \green{Ideal Cliff Walking} and \myred{Car-Racing}, but not in real-world RL environments such as \blue{Breakout} and \orange{Pong}. In \green{Ideal Cliff Walking}, the near-zero values of equinorm, equiangularity, maximal-angle equiangularity, and angular variation indicate that both the \headact\ activations~\footnote{Because activations of Pong and Breakout is continuous and can be sampled in infinity times, mean of them are indeterminated value. We only show mean activations of Ideal Cliff Walking and Car-Racing.} and the \headterm\ weights form a simplex ETF. Based on empirical observations, we propose \NCRLTerm\ as following:
\begin{align}
    \mathbf{M^*} &:= \sqrt{\frac{K}{K-1}}\mathbf{U}\left(\mathbf{I}_K-\frac{1}{K}\mathbf{1}_K\mathbf{1}_K^T\right) \\
    \mathbf{H^*} &:= \mathbf{M^*} \\
    \mathbf{W^*} &:= \mathbf{M^*}
\end{align}
where $\mathbf{H^*}$ and $\mathbf{W^*}$ denote \textit{target} \headact\ activations and \textit{target} \headterm\ weights, respectively.

%% file: sec/4_main.tex
\section{Main Approach}
% ================= our design =================  
\subsection{\NCRLTerm\ Policy Gradient}
The empirical experiments detailed in Sec.\ref{sec:EMP} reveal that a gentle geometry, simplex ETF, emerges in ideal RL environments. Such a structure constitutes an optimal geometric structure for the \headterm\ of optimal trained policy DNNs. This observation thus motivates the ACPG technique, where we explore affixing a synthetic ETF as our \headterm\ and training our PG DNN from this initialization.

% \begin{figure*}[!htb]
% 	\begin{center}
%             \includegraphics[height=7cm]{fig/fig2.pdf}
% 		\caption{ACPG Framework. After observing states from the environment, the backbone layers can be trained to receive input states and the output hidden features aligned with any direction. ETF \headterm\ is a random simplex ETF and is \blue{locked} during training. In contrast, normal DNNs employ a learnable \headterm\ as shown in (a). Illustration of equiangular separation (a) and non-equiangular separation (b) in a 3D space. \NCRLTerm\ reveals the structure in (a), where the weights of \headterm\ and mean \headact\ activations collapse to a simplex equiangular tight frame. As some states are minorly explored and sampling in real RL, their weights lie in a close position, and the separability for these weights and mean activations degrades illustrated in (b). we can see in (b), the green circle shifted to $W_{a4}$ biasly, which shall be performed $W_{a_3}$ in (a).}
% 		\label{fig:FM}
% 	\end{center}
% \end{figure*}

\begin{figure*}[!htb]
    \centering
    \includegraphics[width=\linewidth]{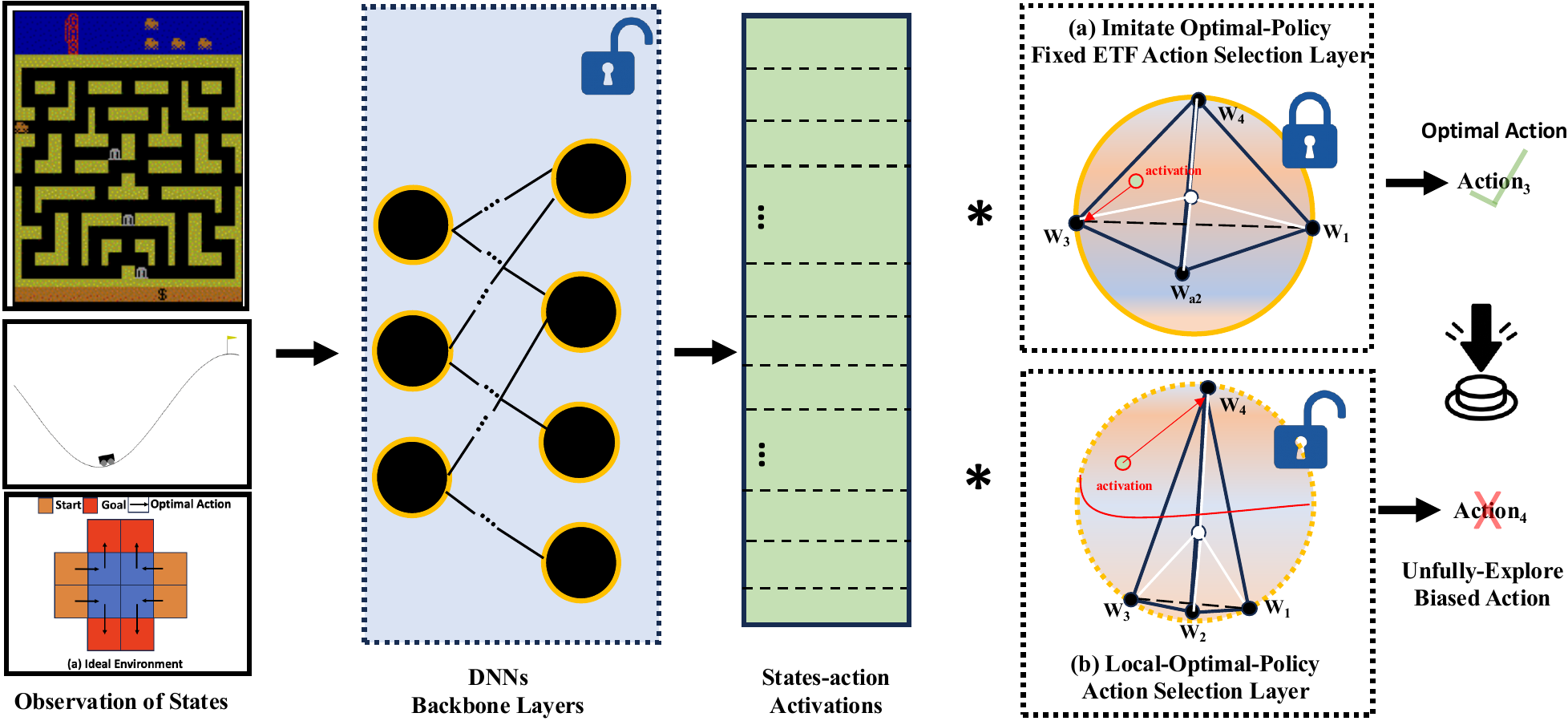}
    \caption{
    ACPG Framework. After observing states from the environment, the backbone layers can be trained to receive input states and the output hidden features aligned with any direction.
    \headterm\ is a random simplex ETF and is \blue{locked} during training.
    By contrast, normal DNNs employ a learnable \headterm\ as shown in (b).
    Illustration of equiangular separation (a) and non-equiangular learnable separation (b). %in a 3D space.
    \NCRLTerm\ reveals (a), where the weights of \headterm\ and mean \headact\ activations collapse to a simplex ETF.
    As some states are less explored than others in realistic environments, their weights lie in a closer to the origin, and the separability for these weights and mean activations degrades, as illustrated in (b).
    We can also see in (b), the green circle is biased toward $w_{a_4}$, which should instead perform ${a_3}$ as in (a).
    }
    \label{fig:FM}
\end{figure*}

Specifically, we initialize the \headterm\ $\mathbf{W}$ as a randomly oriented simplex ETF according to Eq.\eqref{eq:ETFM}, scaled by a constant $\sqrt{E_\mathbf{W}}$ to enforce a fixed $\ell_2$ norm. During training, only \headact\ activations $\mathbf{H}$ are optimized. An overview of our method is illustrated in Fig.\ref{fig:FM}. This selection layer outputs the likelihood of actions, from which an action is sampled and applied to the environment\footnote{The agent's policy may not always select the highest-probability action. For instance, $\epsilon$-greedy exploration sometimes selects a random action to balance exploration and exploitation.}. Although fixing the \headterm\ as an ETF simplifies the learning problem, it brings theoretical merits and higher performance in our experiments. 

% ================= theoretical support =================  
\subsection{Theoretical Support}
In this subsection, we further explore the theoretical aspects of our framework, demonstrating that even in complex, real-world RL environments, %formation of an optimal ETF geometry 
\NCRLTerm\ is also achievable using a fixed ETF \headterm. % The necessity arises from the challenge of analyzing DNNs, which function as highly interactive but non-convex systems. 
To facilitate tractable theoretical analysis, simplification is essential. The layer-peeled model (LPM), as proposed in~\cite{fang2021exploring}, offers an effective simplification by focusing solely on the \headterm\ while disregarding the variable weights of the backbone layers. % Although the LPM is not directly applicable to real application problems, it serves as a useful analytical tool and captures key learning behaviors of the RL DNNs. 
To capture key learning behaviors of the PG, we assume the DNN is modeled as an LPM and $S$ is defined in Sec.\ref{sec:PRE}. We allow $\mathcal S_k \subseteq \mathcal{S}$ and the possibility that the sizes of various $\mathcal S_k$ can be wildly imbalanced. Moreover, the \headterm\ is fixed as a simplex ETF and only \headact\ activations $\mathbf{H}=\{h_s:s\in \mathcal{S}\}$ are optimized~\footnote{$h_{s_k}$ is the \headact\ activation whose optimal action is $a_k$.}. In this case, the $J(\mathbf{H})$ in Eq.\eqref{eq:PGLOSS} is represented as: %the following problem:
\begin{align} \label{eq:LPNOW}
    \max_{\mathbf{H}}\quad & J(\mathbf{H})
= \sum_{s \in \mathcal{S}} d^{\pi}(s) \cdot \log \left(\frac{e^{h_{s}^T w_{a^*} } }{\sum_{a \in \mathcal{A}} e^{h_{s}^T w_a}}\right) \Psi^\pi, \\
    \text{s.t.} \quad & ||h_{s}||^2\le E_\mathbf{H}, \ \forall k, k'\in[1,K], s \in \mathcal{S},
\end{align}
where $a^*$ represents the optimal action of state $s$ and $d^{\pi}(s)$ represents the stationary distribution of the states. Note $\mathbf{W}^*=[w_{1}^*; \dots; w_{K}^*]$ is the \headterm\ as a simplex ETF with:
\begin{align} \label{eq:LPNOWCON}
	{w_k^*}^Tw^*_{k'}=E_\mathbf{W}\left(\frac{K}{K-1}\delta_{k,k'}-\frac{1}{K-1}\right),
        \quad \forall k, k'\in[1,K],
\end{align}
where $\delta_{k,k'}$ equals to $1$ when $k=k'$ and $0$ otherwise. 

%We observe that this practice decouples the multiplied learnable variables $\h_{s}$ and $\w_k$ of LPM in Eq.\ref{eq:PGLOSS}, and makes the model in Eq. (\ref{eq:LPNOW}) a convex problem that is more mathematically tractable. We term the decoupled LPM in Eq. (\ref{eq:LPNOW}) as \textbf{DLPM} for short. 
Despite the state set $\mathcal{S}$ being partially explored and the subset of states $\mathcal S_k$ corresponding to each optimal action $k$ being imbalanced, Eq.\eqref{eq:LPNOW} still admits a global optimum, as shown in the following:
\begin{theorem} \label{the:ETF}
Whether or not 1) the states are fully explored; or 2) the action-state subsets $\mathcal S_k$ are mutually balanced; or 3) the action-state subsets $\mathcal S_k$ are fully-sized ($\mathcal S_k = \mathcal S$); any global minimizer $\mathbf{H}^*=[h^*_{s_k}: s_k \in \mathcal{S}]$ of Eq.\eqref{eq:LPNOW} converges to a simplex ETF with the same direction as $\mathbf{W}^*$ with a length of $\sqrt{E_\mathbf{H}}$, \emph{i.e.,}
\begin{equation} % \label{eq:LPNOWOPT}
{h_{s_k}^*}^Tw^*_{k'}=\sqrt{E_\mathbf{H}E_\mathbf{W}}\left(\frac{K}{K-1}\delta_{k,k'}-\frac{1}{K-1}\right), \quad \forall\ 1\le k, k' \le K, s \in \mathcal{S}
\end{equation}
where $s_k$ represents a state $s \in \mathcal{S}_k$. This indicates that \ncrlTerm\ emerges irrespective of the aforementioned conditions. This can be proven by converting to a conditional Lagrangian problem and applying Karush-Kuhn-Tucker (KKT) conditions. Please refer to Appendix \ref{app:PROOF} for our proof.
\end{theorem}

Recalling the center computation in Eq.\eqref{eq:RLO}, we analyze the gradient of $J(\theta)$ w.r.t. the activation vector $h_s$ (recall this is short for $h(s;\theta)$) for each state $s$. Then the gradient of $J(\theta)$ has:
\begin{align*}
    \nabla_\theta J(\theta) &\propto \sum_{s \in \mathcal{S}} d^\pi(s) \mathbb{E}_\pi [\psi^\pi \nabla_\theta \ln \pi_\theta(a \vert s)] = \sum_{s \in S} d^\pi(s) \left(w_{k^*} - \frac{\sum_{j=1}^K w_j e^{h_s^T w_j}}{\sum_{j=1}^K e^{h_s^T w_j}}\right) \nabla_\theta h_s \psi^\pi \\
    & = \underbrace{\sum_{s \in \mathcal{S}} d^\pi(s) \left(\psi^{+}w_{k^*}(1 - \pi_\theta(a_{k^*} \vert s)\right)\nabla_\theta h_s}_{\text{$a_{k^*}$ is optimal for $s$}} - \notag \quad \underbrace{\sum_{s \in \mathcal{S}} d^\pi(s) \sum_{k = 1, \neq k^*}^{K} \left(\psi^{-}w_k\pi_\theta(a_k \vert s)\right)\nabla_\theta h_s}_{\text{$a_k$ is not optimal for $s$}}\\
\end{align*}
where  $d^\pi(s)$ is stationary distribution of states. $\pi(a_k | s)$ is the policy probability of taking action $a_k$ at state $s$. As usual, Monte Carlo estimation is applied to approximate the formula inside the expected value. Note that $a_{k^*}$ indicates the optimal action for state $s$, which brings the highest return $\psi^+$, and other actions $a_k, k \neq k^*$ can bring lower return $\psi^-$.  %and takes similar form as \eqref{eq: PG-loss}.

It reveals that the gradient can be decomposed into two parts: the ``optimal-action'' part pulls gradient towards optimal action, $w_{a_k^*}$; while ``taking other actions'' part contains negative terms and pushes the gradient to the suboptimal \headterm. This implies that once fully trained, the \headact\ activations will align with their corresponding \headterm, forming a simplex ETF.

%% file: sec/5_exp.tex
\section{Experiments} \label{sec:EXP}
% In this section, 
We investigate the performance improvement of \ours\ in Sec.\ref{sec:RES} and analyze how does the exploration affect \ours. Whether \NCRLTerm\ forms after applying \ours\ is shown in Sec.\ref{sec:ANA}.% Our experiments %mainly conducted on real-world RL environments, namely 

% are conducted on gym~\cite{brockman2016openai} environments. % Each result is the average over three random seeds and we choose the best from three learning rates. The code will be released. 

%The \href{https://github.com/zzz0906/NeuralCollapse4RL}{code }is available.

\subsection{Main Results} \label{sec:RES}
% \noindent\textbf{Comparative Evaluation.} In Tab.\ref{tb:CNN}, we summarize the comparison results for two classic environments: discrete Cart Pole (Car.) and discrete Car-racing (Rac.), three Atari environments~\cite{Bellemare_2013}: EnduroNoFrameskip-v4 (End.), QbertNoFrameskip-v4 (Qbe.) and PongNoFrameskip (Pon.) with three popular PG algorithms: PPO~\cite{schulman2017proximal}, TRPO~\cite{schulman2015trust}, A3C~\cite{mnih2016asynchronous}, and REINFORCE~\cite{sutton1999policy}. To assess the generality and effectiveness of \ours, two classes of Policy DNNs, multilayer perceptron (MLP)~\cite{gardner1998artificial} and convolutional neural network (CNN)~\cite{albawi2017understanding} are employed. We applied \ours\ to our existing DNNs while keeping the underlying policy-gradient algorithms unchanged. \ours\ delivered as much as a 64\% performance gain in End., and across most CNN configurations it produced improvements of 5 – 15\%. Notably, when combined with TRPO in the QBE., \ours\ actually yielded a slight performance drop. We believe this is due to sampling bias: the highly imbalanced state–reward distribution makes it difficult to construct \NCRLTerm. The reward improvement and the reduced number of epochs showcase \ours's enhanced effectiveness and convergence speed.
\noindent\textbf{Comparative Evaluation.} In Tab.\ref{tb:CNN}, we summarize the comparison results for two classic environments: discrete Cart Pole (Car.) and discrete Car-racing (Rac.), three Atari environments~\cite{Bellemare_2013}: EnduroNoFrameskip-v4 (End.), QbertNoFrameskip-v4 (Qbe.) and PongNoFrameskip (Pon.) with four popular PG algorithms: REINFORCE~\cite{sutton1999policy}, TRPO~\cite{schulman2015trust}, PPO~\cite{schulman2017proximal} and A3C~\cite{mnih2016asynchronous}. To assess the generality and effectiveness of \ours, two classes of Policy DNNs, multilayer perceptron (MLP)~\cite{gardner1998artificial} and convolutional neural network (CNN)~\cite{albawi2017understanding} are employed. We applied \ours\ to our existing DNNs while keeping the underlying policy-gradient algorithms unchanged. Across most CNN configurations, \ours\ produced improvements of 5 – 15\%. Notably, when combined with TRPO in the QBE., \ours\ actually yielded a slight performance drop. We believe this is due to sampling bias: the highly imbalanced state–reward distribution makes it difficult to construct \NCRLTerm. From these experiments we observe:
\begin{enumerate}
    \item Consistent performance gains: Across all four environments, ACPG maintains or increases its improvement.
    \item Reduced variability: Most of the standard deviation of ACPG’s results is uniformly lower than the baseline, indicating smoother and more reproducible convergence.
\end{enumerate}
The reward improvement and the reduced number of epochs showcase \ours's enhanced effectiveness and convergence speed.

% | Orig. (ACPG) |

% --- put these once in your preamble if you like ---
% \usepackage{multirow,booktabs}
% ---------------------------------------------------

\begin{table*}[htb]
\centering
\large
\renewcommand{\arraystretch}{1.5}
\begin{adjustbox}{max width=\textwidth,keepaspectratio}
\begin{tabular}{|c|c|cccccccc|cccccccc|}
\hline
\multirow{3}{*}{Alg.} & \multirow{3}{*}{Env.} & \multicolumn{8}{c|}{MLP} & \multicolumn{8}{c|}{CNN} \\ \cline{3-18}
& & \multicolumn{2}{c|}{Best} & \multicolumn{2}{c|}{Final} & \multicolumn{2}{c|}{Stop} & \multicolumn{2}{c|}{Final Std}
& \multicolumn{2}{c|}{Best} & \multicolumn{2}{c|}{Final} & \multicolumn{2}{c|}{Stop} & \multicolumn{2}{c|}{Final Std} \\ \cline{3-18}
& & \multicolumn{1}{c|}{Org.} & \multicolumn{1}{c|}{ACPG} & \multicolumn{1}{c|}{Org.} & \multicolumn{1}{c|}{ACPG} & \multicolumn{1}{c|}{Org.} & \multicolumn{1}{c|}{ACPG} & \multicolumn{1}{c|}{Org.} & ACPG
& \multicolumn{1}{c|}{Org.} & \multicolumn{1}{c|}{ACPG} & \multicolumn{1}{c|}{Org.} & \multicolumn{1}{c|}{ACPG} & \multicolumn{1}{c|}{Org.} & \multicolumn{1}{c|}{ACPG} & \multicolumn{1}{c|}{Org.} & ACPG \\ \hline

\multirow{2}{*}{Rei.} & Car.
& \multicolumn{1}{c|}{500} & \multicolumn{1}{c|}{500}
& \multicolumn{1}{c|}{500} & \multicolumn{1}{c|}{500}
& \multicolumn{1}{c|}{33} & \multicolumn{1}{c|}{\textbf{20 (43\%+)}}
& \multicolumn{1}{c|}{2.47} & \multicolumn{1}{c|}{\textbf{1.54}}
& \multicolumn{1}{c|}{500} & \multicolumn{1}{c|}{500}
& \multicolumn{1}{c|}{500} & \multicolumn{1}{c|}{500}
& \multicolumn{1}{c|}{65} & \multicolumn{1}{c|}{\textbf{15 (76\%+)}}
& \multicolumn{1}{c|}{4.41} & \textbf{1.77} \\ \cline{2-18}

& Rac.
& \multicolumn{1}{c|}{202.25} & \multicolumn{1}{c|}{\textbf{233.16 (15\%+)}}
& \multicolumn{1}{c|}{158.72} & \multicolumn{1}{c|}{\textbf{196.15 (23\%+)}}
& \multicolumn{1}{c|}{100} & \multicolumn{1}{c|}{100}
& \multicolumn{1}{c|}{59.04} & \multicolumn{1}{c|}{32.71}
& \multicolumn{1}{c|}{287.45} & \multicolumn{1}{c|}{\textbf{376.63 (31\%+)}}
& \multicolumn{1}{c|}{188.94} & \multicolumn{1}{c|}{\textbf{284.73 (50\%+)}}
& \multicolumn{1}{c|}{100} & \multicolumn{1}{c|}{100}
& \multicolumn{1}{c|}{104.43} & 117.72 \\ \hline

\multirow{3}{*}{PPO} & End.
& \multicolumn{1}{c|}{941} & \multicolumn{1}{c|}{\textbf{1109 (17\%+)}}
& \multicolumn{1}{c|}{884} & \multicolumn{1}{c|}{\textbf{1002 (13\%+)}}
& \multicolumn{1}{c|}{500} & \multicolumn{1}{c|}{500}
& \multicolumn{1}{c|}{143.23} & \multicolumn{1}{c|}{96.41}
& \multicolumn{1}{c|}{1481} & \multicolumn{1}{c|}{\textbf{1509 (2\%+)}}
& \multicolumn{1}{c|}{1483} & \multicolumn{1}{c|}{\textbf{1563 (5\%+)}}
& \multicolumn{1}{c|}{100} & \multicolumn{1}{c|}{100}
& \multicolumn{1}{c|}{15.43} & 20.28 \\ \cline{2-18}

& Qbe.
& \multicolumn{1}{c|}{8685} & \multicolumn{1}{c|}{\textbf{10431 (20\%+)}}
& \multicolumn{1}{c|}{6747} & \multicolumn{1}{c|}{\textbf{9450 (40\%+)}}
& \multicolumn{1}{c|}{500} & \multicolumn{1}{c|}{500}
& \multicolumn{1}{c|}{3584.31} & \multicolumn{1}{c|}{1497.3}
& \multicolumn{1}{c|}{14852} & \multicolumn{1}{c|}{\textbf{15983 (7\%+)}}
& \multicolumn{1}{c|}{14801} & \multicolumn{1}{c|}{\textbf{15779 (6\%+)}}
& \multicolumn{1}{c|}{100} & \multicolumn{1}{c|}{\textbf{93 (7\%+)}}
& \multicolumn{1}{c|}{103.37} & 91.40 \\ \cline{2-18}

& Pon.
& \multicolumn{1}{c|}{10.9} & \multicolumn{1}{c|}{\textbf{14.7 (26\%+)}}
& \multicolumn{1}{c|}{7.8} & \multicolumn{1}{c|}{\textbf{8.7 (12\%+)}}
& \multicolumn{1}{c|}{500} & \multicolumn{1}{c|}{500}
& \multicolumn{1}{c|}{4.95} & \multicolumn{1}{c|}{4.21}
& \multicolumn{1}{c|}{20.1} & \multicolumn{1}{c|}{\textbf{21.1 (5\%+)}}
& \multicolumn{1}{c|}{19.8} & \multicolumn{1}{c|}{\textbf{20.9 (5\%+)}}
& \multicolumn{1}{c|}{100} & \multicolumn{1}{c|}{\textbf{89 (11\%+)}}
& \multicolumn{1}{c|}{3.41} & 2.28 \\ \hline

\multirow{3}{*}{TRPO} & End.
& \multicolumn{1}{c|}{787} & \multicolumn{1}{c|}{\textbf{1290 (64\%+)}}
& \multicolumn{1}{c|}{421} & \multicolumn{1}{c|}{\textbf{1257 (199\%+)}}
& \multicolumn{1}{c|}{500} & \multicolumn{1}{c|}{500}
& \multicolumn{1}{c|}{--} & \multicolumn{1}{c|}{--}
& \multicolumn{1}{c|}{1243} & \multicolumn{1}{c|}{\textbf{1473 (19\%+)}}
& \multicolumn{1}{c|}{1109} & \multicolumn{1}{c|}{\textbf{1309 (18\%+)}}
& \multicolumn{1}{c|}{100} & \multicolumn{1}{c|}{100}
& \multicolumn{1}{c|}{--} & -- \\ \cline{2-18}

& Qbe.
& \multicolumn{1}{c|}{13170} & \multicolumn{1}{c|}{13007 (-0.01\%)}   % slight drop
& \multicolumn{1}{c|}{12954} & \multicolumn{1}{c|}{8971 (31\%-)}
& \multicolumn{1}{c|}{500} & \multicolumn{1}{c|}{500}
& \multicolumn{1}{c|}{--} & \multicolumn{1}{c|}{--}
& \multicolumn{1}{c|}{13787} & \multicolumn{1}{c|}{\textbf{14205 (3\%+)}}
& \multicolumn{1}{c|}{13184} & \multicolumn{1}{c|}{12589 (4\%-)}
& \multicolumn{1}{c|}{100} & \multicolumn{1}{c|}{\textbf{98 (2\%+)}}
& \multicolumn{1}{c|}{--} & -- \\ \cline{2-18}

& Pon.
& \multicolumn{1}{c|}{10.4} & \multicolumn{1}{c|}{\textbf{13.5 (30\%+)}}
& \multicolumn{1}{c|}{6.9} & \multicolumn{1}{c|}{\textbf{7.7 (12\%+)}}
& \multicolumn{1}{c|}{500} & \multicolumn{1}{c|}{500}
& \multicolumn{1}{c|}{--} & \multicolumn{1}{c|}{--}
& \multicolumn{1}{c|}{20.1} & \multicolumn{1}{c|}{\textbf{20.9 (4\%+)}}
& \multicolumn{1}{c|}{20.1} & \multicolumn{1}{c|}{\textbf{20.9 (4\%+)}}
& \multicolumn{1}{c|}{99} & \multicolumn{1}{c|}{\textbf{89 (10\%+)}}
& \multicolumn{1}{c|}{--} & -- \\ \hline

\multirow{2}{*}{A3C} & Car.
& \multicolumn{1}{c|}{500} & \multicolumn{1}{c|}{500}
& \multicolumn{1}{c|}{500} & \multicolumn{1}{c|}{500}
& \multicolumn{1}{c|}{100} & \multicolumn{1}{c|}{\textbf{92 (8\%+)}}
& \multicolumn{1}{c|}{0.23} & \multicolumn{1}{c|}{0}
& \multicolumn{1}{c|}{500} & \multicolumn{1}{c|}{500}
& \multicolumn{1}{c|}{500} & \multicolumn{1}{c|}{500}
& \multicolumn{1}{c|}{100} & \multicolumn{1}{c|}{\textbf{43 (57\%+)}}
& \multicolumn{1}{c|}{14.35} & 0 \\ \cline{2-18}

& Rac.
& \multicolumn{1}{c|}{104.38} & \multicolumn{1}{c|}{\textbf{273.22 (161\%+)}}
& \multicolumn{1}{c|}{113.54} & \multicolumn{1}{c|}{\textbf{190.28 (67\%+)}}
& \multicolumn{1}{c|}{100} & \multicolumn{1}{c|}{100}
& \multicolumn{1}{c|}{11.24} & \multicolumn{1}{c|}{113.55}
& \multicolumn{1}{c|}{174.28} & \multicolumn{1}{c|}{\textbf{285.48 (64\%+)}}
& \multicolumn{1}{c|}{191.52} & \multicolumn{1}{c|}{\textbf{241.65 (26\%+)}}
& \multicolumn{1}{c|}{100} & \multicolumn{1}{c|}{100}
& \multicolumn{1}{c|}{58.90} & 44.62 \\ \hline

\end{tabular}
\end{adjustbox}
\vspace{10pt}
\caption{Comparative analysis of \ours\ performance against original algorithm (Org.) in gym. A detailed comparison of the best and final reward metrics achieved by MLP and CNN architectures, with and without \ours\ integration. ``Best'' describes the highest reward during the whole training. ``Final'' is the reward at the last epoch. “Stop” denotes the early-stop epoch at which the reward has stabilized and the predefined performance target has been reached. "Std" describes the standard deviation of the final results over 20 runs with different seeds.}
\label{tb:CNN}
\end{table*}

% \begin{table}[!htb]
% 	\caption{Performance Comparison of MLP with and without \ours\ Integration Over Varying Steps in the Cart Pole (Car.) Environment.}
% 	\label{tb:REW}
% 	\centering
% 	\begin{tabular}{p{0.5cm}|p{0.5cm}|c|cc}
% 		\toprule
% 		Env. & Alg. & \multicolumn{1}{c|}{steps} & \makecell[c]{MLP \\ Final } & \makecell[c]{MLP+\ours\ \\ Final} \\
% 		\midrule
%         \multirow{6}{*}{Car.} & \multirow{6}{*}{Rei.} & 1(100steps) & 17.58 & 20.50 \\
% 		&& 2(200steps) & 35.79 & 42.46 \\
% 		&& 10(1000steps) & \textbf{\red{500.00}} & 14.00 \\
%         && 25(2500steps) & 500.00 & 500.00 \\
%         && 50(5000steps) & \textbf{\red{500.00}} & 9.0 \\
%         && 200(20000steps) & 500.00 & 500.00 \\
%         \bottomrule
% 	\end{tabular}
% \end{table}

\noindent\textbf{Performance and Convergence.} 
Tab.\ref{tb:CNN} demonstrates that \ours\ markedly accelerates convergence, enhances robustness, and elevates cumulative rewards. In the more challenging Atari domains, \ours\ outperforms standard policy-gradient baselines by margins of 199\%, 64\%, 19\% and 18\%. In the Car. environments, both \ours\ and the original algorithm attain the maximal reward, reflecting these tasks’ relative simplicity. But thanks to simplicity, we can closely analyze convergence differences in detailed reward-per-step time-series plots comparing \ours\ with the original methods. In the Rac environment, \ours\ achieves up to \textbf{161\%} improvement over original A3C.
% As we can see from ~\ref{tb:CNN}, \ours\ demonstrates a remarkable gain in rewards and a faster convergence rate. In the RL training process, even though the mean return rewards may reach the expected threshold, the performance with PPO+\ours\ continues to improve as training progresses. 

\begin{figure*}[!htb]
    \centering
    \includegraphics[width=\linewidth]{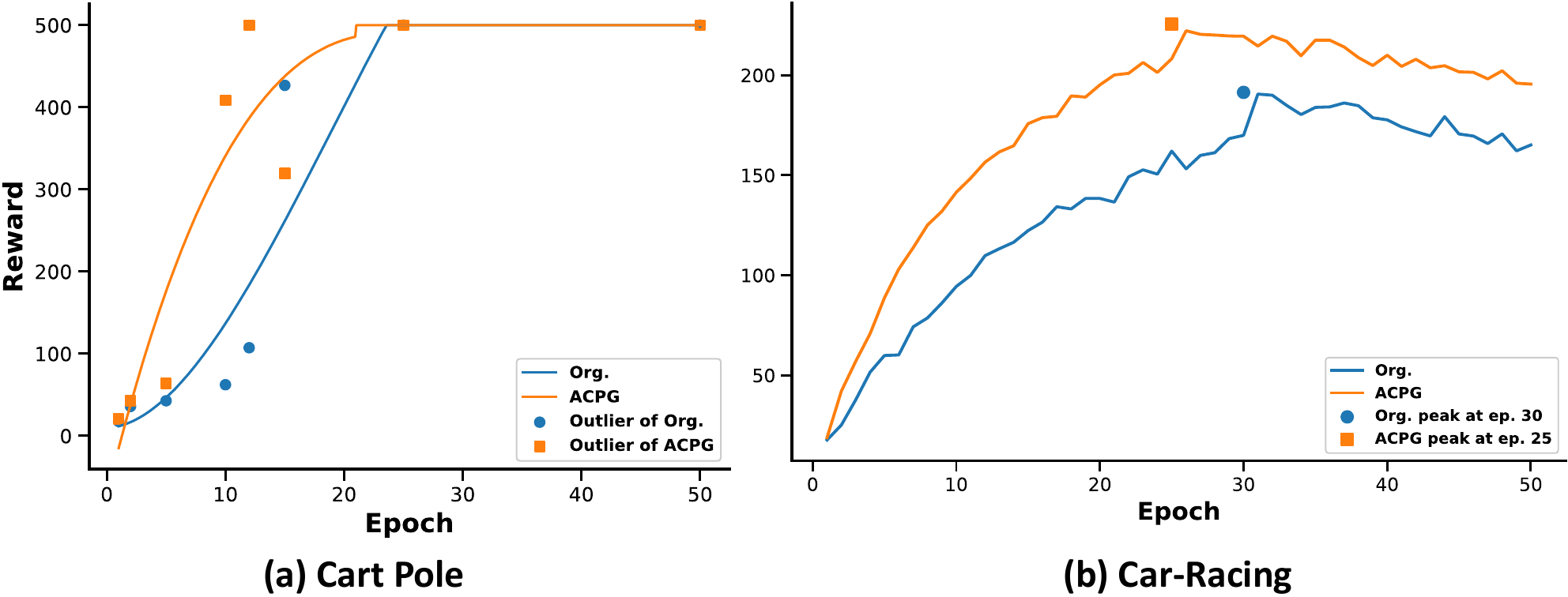}
    \caption{Convergence reward curve of Reinforce (Org.) and \ours-augmented in the Cart Pole (Car.) and Car-Racing (Rac.) Environment. The result is averaged based on multiple experimental runs, each initialized with a distinct random seed. Outliers represent max/min reward achieved.}  %Because of aggregation from multiple independent trials, a handful of these markers lie off the main trend, seed‐induced outliers.}
    \label{fig:COV}
    %\vspace{-10pt}
\end{figure*}

Fig.\ref{fig:COV} highlights \ours's dramatic acceleration: by epoch 12, \ours\ already reaches a nearly 12$\times$ improvement. Across key milestones, \ours\ consistently delivers more than 8$\times$ higher rewards and converges in far fewer epochs. In the Car-Racing environment, \ours\ similarly outpaces the baseline, achieving faster convergence and higher peak performance. %These gains originate from 
\ours’s fixed ETF \headterm\ %, which 
precisely aligns activations with the \NCRLTerm\ targets. By contrast, conventional learnable layers frequently become trapped in local minima and require much longer training to disentangle state–action embeddings. % Consequently, \ours\ more faithfully approximates the optimal policy, inherently balancing exploration and exploitation to drive both faster and more robust learning.

The remarkable improvement underscores \ours’s superior balance of exploration and exploitation: by actively probing under‐explored states, it trains the \headterm\ activations $h$ to better emulate the optimal policy. In comparison, baseline methods show pronounced performance volatility—especially in rare or typical states—because they often settle into local minima, which works well for familiar scenarios but leads to suboptimal action choices. % when faced with less common situations.

\begin{table}[!htb]
    \centering
    \small
    \begin{tabular}{c|c|c|c|cc}
        \toprule
        Alg & Env. & Epochs & $\epsilon$ & \makecell[c]{CNN \\ Best } & \makecell[c]{CNN+\ours \\ Best} \\
        \midrule
        \multirow{4}{*}{PPO} & \multirow{4}{*}{Pon.} & \multirow{4}{*}{100} & 0.000 & 3.1 & 2.1 \\
        &&& 0.001 & 7.1 & 7.2 \\
        &&& 0.010 & 19.8 & 20.9 \\
        &&& 0.100 & 18.5 & 21.6 \\
        &&& 1.000 & 1.7 & 3.2 \\
        \bottomrule
    \end{tabular}
    \vspace{10pt}
    \caption{Impact of exploration for \ours\ in Pon. environment. $\epsilon$-greedy is integrated into PPO.}
    \label{tb:UNS}
    % \vspace{-15pt}
\end{table}

\noindent\textbf{Affect of Exploration Strategy.} 
%Consistent with standard RL training practices as outlined by Sutton~\cite{sutton1999policy}, we introduce a probability-based mechanism to randomly select actions, 
% To understand the effect of exploration to \NCRLTerm, epsilon-greedy strategy~\cite{sutton1999policy} is integrated into PPO, even though PPO contains exploration probability from ``softmax''. Tab.~\ref{tb:UNS} illustrates that, over 100 epochs, a lower epsilon results in less exploration where we cannot see performance gain if we force to apply \ours\. In contrast, when epsilon is higher, \ours\ %which emulates the optimal policy, requires exploration to converge, especially with limited steps and without any exploration, convergence is not achieved. Notably, as Epsilon approaches \textbf{0.001}, \ours\+PPO 
% demonstrates significantly performance improvements. We conjecture that because \ours\ emulates the optimal policy, it requires more exploration to converge. But at an epsilon value of 1.0, where the sampling encompasses a highly diverse range of explored states, both methods diverge. \ours\ exhibits superior adaptability with higher Epsilon values, ensuring convergence amid a plethora of unseen states and maintaining robustness.
To investigate exploration’s impact on \NCRLTerm, we incorporated an $\epsilon$-greedy strategy~\cite{sutton1999policy} into PPO—despite PPO’s own softmax-based exploration. Tab.\ref{tb:UNS} shows that, over 100 epochs, very low $\epsilon$ (i.e., minimal exploration) produces no performance gain when \ours\ is applied. By contrast, as $\epsilon$ increases—especially around \textbf{0.001} — the \ours\ yields significant improvements. We hypothesize that, because \ours\ emulates the optimal policy, it demands adequate exploration to converge; but at $\epsilon$ = 1.0, where sampling covers an excessively broad state distribution, both methods diverge. \ours\ exhibits superior adaptability with higher Epsilon values, ensuring convergence amid a plethora of unseen states and maintaining robustness.

% TODO:
% Tab.\ref{tb:REW}

\subsection{Does \NCRLTerm\ exist in \ours?} \label{sec:ANA}
\begin{figure}
    \centering
    \includegraphics[width=\linewidth]{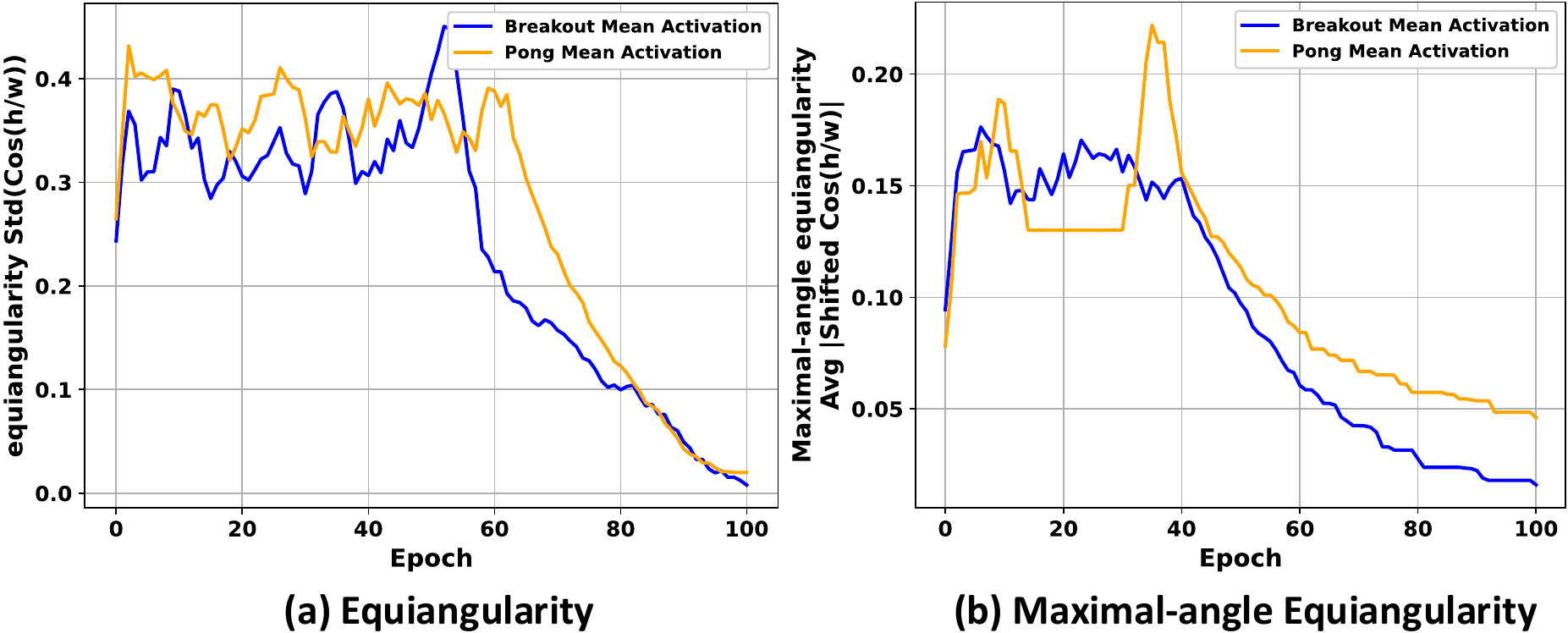}
    \caption{In Breakout and Pong, after applying \ours, mean \headact\ activations approach equiangularity and maximal-angle equiangularity. Y-axis: $\text{Std}_{k,{k'}\neq k}(\cos_{h/w}(k, {k'}))$ and $\text{Avg}_{k,{k'}}|\cos_{h/w}(k, {k'})+ 1/(K-1)|$. Activations from training samples are averaged per epoch.}
    \label{fig:REEVAL}
    %\vspace{-15pt}
\end{figure}

In Sec.\ref{sec:EMP}, \NCRLTerm\ does not naturally occur in Breakout and Pong. But under \ours, the mean \headterm\ activations trend towards equiangularity and maximal‐angle equiangularity in Fig.\ref{fig:REEVAL}. % However, imbalanced sampling distorts this ideal structure: 
While a fixed ETF \headterm\ formally defines \NCRLTerm, it may not be the most effective configuration. In practice, settling at a saddle‐point solution—rather than the global optimum—can yield more robust performance.

% \subsection{\NCRLTerm\ of Large Language Model fine-tuning in RLHF} \label{sec:RLHF}

%% file: sec/6_rel.tex
\section{Related  Work} \label{sec:REL}

\paragraph{Policy Gradient} PG methods have garnered attention for their ability to scale gracefully to large spaces and incorporate deep networks as function approximators~\cite{silver2014deterministic, schulman2015trust}. REINFORCE~\cite{sutton1999policy} relies on estimating returns through Monte-Carlo methods, while Actor-Critic methods~\cite{konda1999actor} approximate returns using another value network. Both of them are on-policy, with corresponding off-policy versions utilizing collected weighted trajectory rewards as returns. Asynchronous Advantage Actor-Critic (A3C)~\cite{mnih2016asynchronous} is an Actor-Critic method that generates sample batches by running multiple actors in parallel. DPG and DDPG~\cite{silver2014deterministic, lillicrap2015continuous} model the policy as a deterministic decision, where the action is deterministically chosen. ~\cite{schulman2015trust}, ~\cite{schulman2017proximal}, ~\cite{rafailov2023direct} proposed TRPO, PPO and DPO methods, which enforce a KL divergence constraint on policy update, helping avoid excessive parameter updates. While these previous PG algorithms study the effect of return function including update frequency, loss function and stability, they ignore the optimal geometric structure when applying DNNs.

\paragraph{Neural Collapse} NC is a phenomenon observed in the classification problem, which inspires us to explore neural network geometric structure in RL.~\cite{papyan2020prevalence} first discovered the NC phenomenon. At the TPT, a classification model trained on a balanced dataset exhibits the collapse of last-layer features into their within-class centers. These centers and classifiers will form a simplex ETF. Due to its elegant geometric property, recent studies prove it with the assumption of last-layer features and the last-layer as independent variables under CE loss~\cite{fang2021exploring, lu2020neural} and MSE loss~\cite{mixon2020neural, han2021neural}. Some studies try to induce neural collapse in imbalanced learning for better accuracy of minority classes~\cite{liu2023inducing, liang2023inducing, yang2022inducing}. Besides, NC has also been observed during LLM training process, which is linked to increasing generalization ~\cite{wu2024linguistic}.  By contrast, we discover the structures of the last-layer feature means and \headterm\ in PG and propose \ours, which allows the policy to imitate the optimal policy.

\textbf{Optimality and Convergence Analysis}
Policy training requires the action distribution to converge either to the global optima or a stationary point~\cite{xiong2022deterministic}. However, understanding of the optimal geometric structure is limited. Theorem 4.2 of ~\cite{zhang2019convergence} establishes the asymptotic convergence of random-horizon PG algorithm. Additionally, \cite{zhang2020global} proposes \textit{Modified Random-Horizon Policy Gradient} (MRPG) algorithm to approximately converge to a second-order stationary point. To ensure convergence, objective functions such as regularizing entropy are widely accepted and proven to be less likely converge to local optima, offering a robust optimization ~\cite{entropy_optimization,regularized}. The deterministic policy can be optimized with a guaranteed stationary solution~\cite{deterministic}, but this usually  converges in a local optimum in practice. When there is a gradient dominance condition, ~\cite{bhandari2019global, bhandari2021linear} studied linear convergence rate for PG methods. Nevertheless, all these works only study how and when the agent model can converge to the stationary point or global optimal, but ignore the optimal geometric structure of model in TPT.

\textbf{Exploration and Exploitation}
Several theoretical analyses in reinforcement learning aim to comprehend how to explore and exploit environmental data for improving model training. Optimism in the Face of Uncertainty (OFU) is a method based on confidence~\cite{strehl2005theoretical, regretbound_analysis, fix-horizon} or bonus~\cite{strehl2006pac, jin2018q}. While they have strong theoretical analysis and guarantees, applying them in deep RL presents challenges due to the need for preserving visits for each state in a tabular way. Specifically to PG, theoretical analysis of entropy maximization indicates increasing action entropy only promotes undirected exploration. Therefore, it is suggested to balance the random choice and the output of policy model to encourage more exploration while maintaining exploiting existing network.
%Several alternative methods has been proposed for RL are prediction-error-based and informationgain-based require the estimation of some probability density model to encourage agent 'explore' novel state which has lower probability. 

Other related works - \textbf{Reinforcement Learning in Large Language Model} %, \textbf{Optimality and Convergence Analysis} and \textbf{Exploration and Exploitation} are 
is included in appendix~\ref{app:RELATED}

%% file: sec/7_con.tex
\section{Conclusion} \label{sec:CONC}
In this paper, we explore the \ncrlTerm\ for RL PG methods. To the best of our knowledge, this is the first work to investigate the geometric properties of optimal policy DNNs. We begin by empirically observing the \ncrlTerm\ phenomenon and subsequently extend our analysis in real and complex RL environment. In real-world settings, where $\mathcal{S}$ is not fully explored and certain conditions in empirical experiments are not satisfied, %the equiangular and maximally separated structure of 
the \ncrlTerm\ is disrupted. To approach \ncrlTerm, we introduce a synthetic fixed simplex ETF as \headterm. We conclude that our simplified practice even helps to improve the performance of PG methods with no cost. We anticipate that our findings will inspire future theoretical investigations into RL's optimal policy. The limitation of this study is its analysis of \ncrlTerm\ predominantly conducted within certain conditions. Future research should aim to investigate more intricate RL DNNs, more complex environments, and PG loss functions that are more conducive to facilitating \ncrlTerm. 

%% file: sec/9_app.tex
\begin{center}
	\Large \textbf{Imitate Optimal Policy: Prevail and Induce Action Collapse in Policy Gradient\\ \vspace{8pt} \textit{Supplementary Material}}
\end{center}

\section{Neural Collapse Phenomenon} \label{app:NC}

Then the neural collapse (NC) phenomenon can be formally described as ~\cite{papyan2020prevalence}:

\textbf{(NC1)} Collapse of within-class variability: $\Sigma_\mathbf{W}\rightarrow\mathbf{0}$, and $\Sigma_\mathbf{W}:=\mathrm{Avg}_{i,k}\{(h_{k,i}-h_k)(h_{k,i}-h_k)^T\}$, where $h_{k,i}$ is the last-layer feature of the $i$-th sample in the $k$-th class, and $h_k=\mathrm{Avg}_{i}\{h_{k,i}\}$ is the within-class mean of the last-layer features in the $k$-th class;% \emph{i.e.,} $\h_k=\mathrm{Avg}_{i}\{\h_{k,i}\}$;

\textbf{(NC2)} Convergence to a simplex ETF: $\tilde{h}_k = (h_k-h_G)/||h_k-h_G||, k\in[1,K]$, satisfies Eq. (\ref{eq:MIMJ}), where  $h_G$ is the global mean of the last-layer features, \emph{i.e.,} $h_G=\mathrm{Avg}_{i,k}\{h_{k,i}\}$;

\textbf{(NC3)} Self duality: $\tilde{h}_k=w_k/||w_k||$, where $w_k$ is the classifier vector of the $k$-th class;

\textbf{(NC4)} Simplification to the nearest class center prediction: $argmax_k\langle h, w_k\rangle=argmin_k||h-h_k||$, where $h$ is the last-layer feature of a sample to predict for classification.

\section{Proof of Theorem 1} \label{app:PROOF}
We consider the optimization problem:

\begin{equation}
    \max_{\mathbf{H}} J(\mathbf{H})
\end{equation}

where $\mathbf{H} = \{h_{s_k}:s \in \mathcal{S}\}$ represents the activation of states with optimal action $a_k$. Expanding the policy gradient reward function:

\begin{equation}
    \max_{\mathbf{H}}\sum_{k=1}^K\sum_{s_k} d^{\pi}(s_k) \cdot \log \frac{e^{h_{s_k}^T w_{k}}}{\sum_{j=1}^K e^{h_{s_k}^T w_j}}\Psi^\pi
\end{equation}

Rewriting in terms of \headact activations distributions, where we separate the states $s$ into their corresponding optimal action $a_k$:

\begin{equation}
    \max_{\mathbf{H}} \sum_{k=1}^{K} \sum_{i=1}^{n_k} d^{\pi}(s) \cdot\log \frac{e^{h_{i,k}^T w_k}}{\sum_{j=1}^K e^{h_{i,k}^T w_j}}\Psi^\pi
\end{equation}

where $h_{i,k}$ represents the \headact\ activations of the $i^{th}$ state in class $\mathcal{S}_k$, and $n_k = |\mathcal{S}_k|$ represents the cardinality (size) of class $\mathcal{S}_k$. Rearranging (by dividing $e^{h_{i,k}^T w_k}$ from both numerator and denominator):

\begin{equation}
    \max_{\mathbf{H}} \sum_{k=1}^{K} \sum_{i=1}^{n_k} d^{\pi}(s) \cdot \log \frac{1}{1 + \sum_{j \neq k} e^{h_{i,k}^T (w_j - w_k)}}\Psi^\pi
\end{equation}

which is equivalent to minimizing:

\begin{equation}
    \min_{\mathbf{H}} \sum_{j \neq k} e^{h_{i,k}^T (w_j - w_k)}
\end{equation}

where we only focus on optimizing the value of $\mathbf{H} = \{h_{i,k}\}$. Thus, we only maintain the $i^{th}$ sample in class $k$, and write $h_{i,k}$ as $h$ for simplicity. Note that for continuous state space, this formula also holds, since we use integral instead of sum to overwrite $J(\mathbf{H})$. Then we rewrite to:

\begin{equation}
    \min_{h} \sum_{j \neq k} e^{h^T (w_j - w_k)}
\end{equation}

Define the function:

\begin{equation}
    f(h) = \sum_{j \neq k} e^{h^T (w_j - w_k)}
\end{equation}

subject to the constraint:

\begin{equation}
    g(h) = \| h \|^2 - E_\mathbf{H} \leq 0
\end{equation}

Construct the Lagrangian:

\begin{equation}
    \mathcal{L}(h, \lambda) = f(h) + \lambda g(h)
\end{equation}

Using Karush-Kuhn-Tucker (KKT) conditions, we require:

\begin{equation}
    \frac{\partial}{\partial h} \mathcal{L}(h, \lambda) = \sum_{j \neq k} (w_j - w_k) e^{h^T (w_j - w_k)} + 2 \lambda h = 0
\end{equation}

The KKT conditions require:

\begin{equation}
    \lambda g(h) = 0, \quad g(h) \leq 0, \quad \lambda \geq 0
\end{equation}

Note both $f(h)$ and $g(h)$ are convex functions, since they both have non-negative second derivatives (which can be easily verified by the second derivative of $(e^{h^T(w_j-w_k)})'' = (w_j-w_k)^2e^{h^T(w_j-w_k)} \geq 0$ in $f(h)$ and $(h^2)'' = 2 > 0$ in $g(h)$). Since KKT conditions are both necessary and sufficient for convex problems , we proceed with:

\begin{lemma}
    $\sum_{k=1}^{K} w_k = 0$.
\end{lemma}

\begin{proof}
    The Gram matrix for a set of simplex ETF:
    \begin{equation}
        G_{ij} = \begin{cases} E_\mathbf{W}, & i=j \\ -\frac{E_\mathbf{W}}{K-1}, & i \neq j \end{cases}
    \end{equation}
    It can be verified that Gram matrix $G$ has all-ones vector $\mathbf{1}$ as an 0-eigenvector since:
    \begin{equation}
        G \mathbf{1} = E_\mathbf{W} + (K-1)(-\frac{E_\mathbf{W}}{K-1}) = 0
    \end{equation}
    Substituting $G=\mathbf{W}^T\mathbf{W}$ following the definition of simplex ETF we have: 
    \begin{equation}
        G \mathbf{1} = 0 \Rightarrow \mathbf{W}^T \mathbf{W}\mathbf{1} = 0
    \end{equation}
    This can be simplified from how we construct the ETF matrix $\mathbf{W}$. We assume $\mathbf{W}\in R^{d \times K}$, and for any activation dimension $d > K-1$, we add 0 to the row greater than $K-1$ and we keep the elements in first $K-1$ rows in each column to be the tightest simplex ETF (i.e. $K$ simplex ETF vectors in a $K-1$ space). Therefore, proving the sum of \headterm\ to be $0$-vector needs only to be considered in the case where $\mathbf{W}\in R^{(K-1) \times K}$ is in the tightest shape, since the other components with row index greater than $K-1$ are set to $0$. 
    \\
    Therefore, to solve the equation $\mathbf{W}^Ta=0$ for some vector $a\in R^{K-1}$, we have: $a=0$
    since the vector $\mathbf{W}^T$ have full rank in its column, indicating it has only $0$-vector in its null space. This indicates that:
    \begin{equation}
        \mathbf{W}^T \mathbf{W}\mathbf{1} = 0 \Rightarrow \mathbf{W}\mathbf{1} = 0
    \end{equation}
The above equation simplifies to $\sum_{j=1}^{K} w_j = 0$.
    
\end{proof}

\begin{proposition}
    KKT holds based on the condition: 
    \begin{equation}
        h = \sqrt{\frac{E_\mathbf{H}}{E_\mathbf{W}}} w_k,
    \end{equation}
\end{proposition}

\begin{proof}
\begin{equation}
    \sum_{j \neq k} (w_j - w_k) e^{w^T (w_j - w_k)} + 2 \lambda h = 0
\end{equation}

We define:

\begin{equation}
    A = e^{\sqrt{\frac{E_\mathbf{H}}{E_\mathbf{W}}} w_k^T (w_j - w_k)}
\end{equation}

where A is a constant independent of $j$ since $ w_k^Tw_{k'}=-\frac{E_\mathbf{W}}{K-1} $ if $k\ne k'$ and $E_\mathbf{W}$ otherwise. And we rewrite the formula as:

\begin{equation}
    A \sum_{j \neq k} (w_j - w_k) + 2 \lambda h = 0.
\end{equation}

Rearranging and substituting $\sum_{j=1}^{K} w_j = 0$ from the Lemma and condition $h = \sqrt{\frac{E_\mathbf{H}}{E_\mathbf{W}}} w_k$:

\begin{equation}
    2 \lambda \sqrt{\frac{E_\mathbf{H}}{E_\mathbf{W}}} w_k = A K w_k,
\end{equation}

where we pick the optimal $\lambda$ to be:

\begin{equation}
    \lambda = \frac{1}{2} \sqrt{\frac{E_\mathbf{W}}{E_\mathbf{H}}} A K > 0.
\end{equation}

The norm constraint is active, indicating $g(h) \leq 0$ since:
\begin{equation}
    g(h) = \| h \|^2 - E_\mathbf{H} = \frac{E_\mathbf{H}}{E_\mathbf{W}} \| w_k \|^2 - E_\mathbf{W} = 0,
\end{equation}

also \( \lambda = \frac{1}{2} \sqrt{\frac{E_\mathbf{W}}{E_\mathbf{H}}} A K > 0 \), all KKT conditions are satisfied, proving:

\begin{equation}
    h = \sqrt{\frac{E_\mathbf{H}}{E_\mathbf{W}}} w_k.
\end{equation}

is the optimal solution (from our previous argument, in convex case, satisfying KKT condition is necessary and sufficient to show optimal solution of $h$ and $\lambda$).

\end{proof}
The proposition above satisfies the condition we aim to prove in the theorem:

\begin{equation}
    h^T w_{k'} = \sqrt{E_\mathbf{H} E_\mathbf{W}} \left( \frac{K}{K-1} \delta_{k,k'} - \frac{1}{K-1} \right).
\end{equation}

where $h = h_{i,k}$ or $h_{s_k}$ represents the activation of a state in $\mathcal{S}_k$

\qed

Note: Notable differences exist between the loss function in Eq.\eqref{eq:RLO} and Eq.\eqref{eq:PGLOSS}. The key difference is, Eq.\eqref{eq:RLO} follows the state value function as the loss function, which includes all the action choices in the formula. By contrast, Eq.\eqref{eq:PGLOSS} leverages PG, which solely focusing on optimizing the policy probability. Both formulations can serve as valid loss functions for policy training. In fact, the PG formulation is derived from the value function using Monte Carlo estimation, where the optimal action choice acts as an unbiased estimator for the expected value. In this paper, PG is used as the loss function in reaching the Theorem, and value function is expanded to show the phenomenon that activations will be attracted to align with the simplex ETF structure.

\section{Experiments Details} \label{app:EXP}
All experiments were conducted on servers equipped with NVIDIA H100 80 GB GPUs paired with dual Intel Xeon Platinum 8462Y+ processors (2 × 32-core sockets, 64 cores total) and approximately 2 TB of RAM. \\ 
All hyper-parameters are shown as below:
\begin{itemize}
    \item The MLP Classic Control / box2d: MLP hidden sizes: [64, 64], Batch size: 64, Buffer size: 20000, Activation: ReLU, Output: Softmax for policy.
    \item The MLP Atari: Hidden sizes: [N, 256, 128, 256, N] (N is based on the feature dimension after CNN process image), Activation: ReLU, Output: Linear for Q-values.
    \item The CNN: Conv layers: [256->128->32->64->64] channels, Kernel sizes: [8->4->3], Strides: [4->2->1], Hidden size: [512], Then following MLP above.
    \item Ideal Cliff Walking: Epochs: 100/500, Step per epoch: 1000, Reward threshold: 195.
    \item Car-racing: Epochs: 100/500, Step per epoch: 100000, Step per collect: 1000, Repeat per collect: 4, Training num: 10, Test num: 10.
    \item Break-out: Epochs: 100/500, Step per epoch: 100000, Step per collect: 1000, Repeat per collect: 4, Training num: 10, Test num: 10.
    \item End.: Epochs: 100/500, Step per epoch: 100000, Step per collect: 1000, Repeat per collect: 4, Training num: 10, Test num: 10
    \item Qbe.: Epochs: 100/500, Step per epoch: 100000, Step per collect: 1000, Repeat per collect: 4, Training num: 10, Test num: 10
    \item Pon.: Epochs: 100/500, Step per epoch: 100000, Step per collect: 1000, Repeat per collect: 4, Training num: 10, Test num: 10
    \item Reinforce hyper-parameters: Learning rate: (1e-3, 1e-4, 1e-5), Gamma: 0.95, Batch size: 64, Buffer size: 20000, Training episodes per collect: 8, Repeat per collect: 2
    \item PPO hyper-parameters: Learning rate: (8e-5, 2.5e-4, 8e-4) Gamma: 0.99, GAE lambda: 0.95, Clip epsilon: 0.1, Value function coefficient: 0.25, Entropy coefficient: 0.01, Max gradient norm: 0.5, Batch size: 256, Buffer size: 100000, Training episodes per collect: 10, Repeat per collect: 4
    \item TRPO hyper-parameters: Learning rate: (8e-5, 2.5e-4, 8e-4),Gamma: 0.99,Max KL divergence: 0.01, Value function coefficient: 0.5, Entropy coefficient: 0.01, Batch size: 256, Buffer size: 100000, Training episodes per collect: 10, Repeat per collect: 4
    \item Batch size: 64 (classic control) / 256 (Atari) over all environments.
\end{itemize}
All experimental results are average over \textbf{20} random seeds and we choose the best from three learning rates.

% \subsection{Implementation Details}

% \subsection{Training Details}

\section{Related Work} \label{app:RELATED}
\textbf{Reinforcement Learning in Large Language Model}
RL, especially RLHF, has been widely applied to train LLM, where RL favors the process of improving the efficiency of human feedback. Human values and preferences are used in improving the reward function accuracy, where a latent variable framework can be introduced for better encoding ~\cite{sriyash2024}. Even the need of explicit reward function can be replaced by preference data using a classification loss, with the help of \textit{direct preference optimization} (DPO) ~\cite{rafailov2023direct}. Also, safety issue in LLM training and development is currently more and more important, which can be formalized and included in reward function to improve helpfulness and harmlessness, utilizing existing methods in RL to favor the optimization ~\cite{dai2024safe}. Harnessing the advantage of PG, methods such as \textit{contrastive policy gradient} (CoPG) are applied in fine-tuning and output human preferences more effectively ~\cite{flet2024} . These examples demonstrate the importance of RL in LLM training, which indicates the potential application of our ACPG method.

\section{Limitation and Societal Impacts}
\label{app:LIMIT}
The potential limitations of our study may include:
\begin{itemize}
    \item Our supposition of \ncrlTerm\ has only been shown to happen in ideal environments, while others have not yet been tested. We conjecture that if the environment becomes more complex, with more random transition and sparse rewards, this phenomenon may be difficult to observe or even does not exist. Furthermore, what does geometric structure characterize at a saddle (sub-optimal) point?
    \item Theorem 1~\ref{the:ETF} states our goal is optimized when last layer activations match with their preassigned~\headterm, but have not yet been tested whether arbitrary states from the same optimal-action class will finally collapse to corresponding vector in TPT.
    \item Although experiments have demonstrated better adaptability in Sec.\ref{sec:RES} by comparing PPO with fixed \headterm\ or not, other PG methods has not yet been fully tested by applying fixed \headterm. The impact of construction of sampling RL dataset also warrants further investigation.
\end{itemize}

\noindent\textbf{Societal Impacts} Our work introduces \ncrlTerm\ Policy Gradient (ACPG), a broadly applicable paradigm for training discrete-policy RL networks by fixing the \headterm\ to a simplex ETF structure. This approach offers theoretical insights into optimal policy geometry and yields practical benefits—namely faster convergence and more stable learning across a wide range of Gym environments—without altering the core interaction loop or data requirements of existing PG methods. Because ACPG only replaces the \headterm\ with a non‐learnable ETF, it often lowers overall computation and energy usage compared with conventional learnable \headterm, imposing no extra environmental burden. Beyond these efficiency gains, our method does not introduce new concerns in areas such as privacy, public health, fairness, or other societal domains: it remains fully compatible with standard RL safety and ethics frameworks and does not affect the downstream interpretability or deployment of learned policies.

%% file: ICLR_2026.bbl
\begin{thebibliography}{54}
\providecommand{\natexlab}[1]{#1}
\providecommand{\url}[1]{\texttt{#1}}
\expandafter\ifx\csname urlstyle\endcsname\relax
  \providecommand{\doi}[1]{doi: #1}\else
  \providecommand{\doi}{doi: \begingroup \urlstyle{rm}\Url}\fi

\bibitem[Ahmed et~al.(2019)Ahmed, Le~Roux, Norouzi, and
  Schuurmans]{entropy_optimization}
Zafarali Ahmed, Nicolas Le~Roux, Mohammad Norouzi, and Dale Schuurmans.
\newblock Understanding the impact of entropy on policy optimization.
\newblock In \emph{International conference on machine learning}, pp.\
  151--160. PMLR, 2019.

\bibitem[Albawi et~al.(2017)Albawi, Mohammed, and
  Al-Zawi]{albawi2017understanding}
Saad Albawi, Tareq~Abed Mohammed, and Saad Al-Zawi.
\newblock Understanding of a convolutional neural network.
\newblock In \emph{2017 international conference on engineering and technology
  (ICET)}, pp.\  1--6. Ieee, 2017.

\bibitem[Asadi et~al.(2023)Asadi, Sabach, Liu, Gottesman, and
  Fakoor]{asadi2023td}
Kavosh Asadi, Shoham Sabach, Yao Liu, Omer Gottesman, and Rasool Fakoor.
\newblock Td convergence: An optimization perspective.
\newblock \emph{arXiv preprint arXiv:2306.17750}, 2023.

\bibitem[Askell et~al.(2021)Askell, Bai, Chen, Drain, Ganguli, Henighan, Jones,
  Joseph, Mann, DasSarma, et~al.]{RLHF}
Amanda Askell, Yuntao Bai, Anna Chen, Dawn Drain, Deep Ganguli, Tom Henighan,
  Andy Jones, Nicholas Joseph, Ben Mann, Nova DasSarma, et~al.
\newblock A general language assistant as a laboratory for alignment.
\newblock \emph{arXiv preprint arXiv:2112.00861}, 2021.

\bibitem[Bellemare et~al.(2013)Bellemare, Naddaf, Veness, and
  Bowling]{Bellemare_2013}
M.~G. Bellemare, Y.~Naddaf, J.~Veness, and M.~Bowling.
\newblock The arcade learning environment: An evaluation platform for general
  agents.
\newblock \emph{Journal of Artificial Intelligence Research}, 47:\penalty0
  253–279, June 2013.
\newblock ISSN 1076-9757.
\newblock \doi{10.1613/jair.3912}.
\newblock URL \url{http://dx.doi.org/10.1613/jair.3912}.

\bibitem[Bhandari \& Russo(2019)Bhandari and Russo]{bhandari2019global}
Jalaj Bhandari and Daniel Russo.
\newblock Global optimality guarantees for policy gradient methods.
\newblock \emph{arXiv preprint arXiv:1906.01786}, 2019.

\bibitem[Bhandari \& Russo(2021)Bhandari and Russo]{bhandari2021linear}
Jalaj Bhandari and Daniel Russo.
\newblock On the linear convergence of policy gradient methods for finite mdps.
\newblock In \emph{International Conference on Artificial Intelligence and
  Statistics}, pp.\  2386--2394. PMLR, 2021.

\bibitem[Brockman et~al.(2016)Brockman, Cheung, Pettersson, Schneider,
  Schulman, Tang, and Zaremba]{brockman2016openai}
Greg Brockman, Vicki Cheung, Ludwig Pettersson, Jonas Schneider, John Schulman,
  Jie Tang, and Wojciech Zaremba.
\newblock Openai gym.
\newblock \emph{arXiv preprint arXiv:1606.01540}, 2016.

\bibitem[Cobbe et~al.(2020)Cobbe, Hilton, Klimov, and
  Schulman]{cobbe2020phasic}
Karl Cobbe, Jacob Hilton, Oleg Klimov, and John Schulman.
\newblock Phasic policy gradient, 2020.

\bibitem[Dai et~al.(2024)Dai, Pan, Sun, Ji, Xu, Liu, Wang, and
  Yang]{dai2024safe}
Josef Dai, Xuehai Pan, Ruiyang Sun, Jiaming Ji, Xinbo Xu, Mickel Liu, Yizhou
  Wang, and Yaodong Yang.
\newblock Safe rlhf: Safe reinforcement learning from human feedback.
\newblock In \emph{Proceedings of the International Conference on Learning
  Representations (ICLR)}, 2024.
\newblock URL \url{https://openreview.net/forum?id=TyFrPOKYXw}.

\bibitem[Dann \& Brunskill(2015)Dann and Brunskill]{fix-horizon}
Christoph Dann and Emma Brunskill.
\newblock Sample complexity of episodic fixed-horizon reinforcement learning.
\newblock \emph{Advances in Neural Information Processing Systems}, 28, 2015.

\bibitem[Demin(2009)]{demin2009cliff}
Vladimir Demin.
\newblock Cliff walking problem, 2009.

\bibitem[Fan et~al.(2020)Fan, Wang, Xie, and Yang]{fan2020theoretical}
Jianqing Fan, Zhaoran Wang, Yuchen Xie, and Zhuoran Yang.
\newblock A theoretical analysis of deep q-learning.
\newblock In \emph{Learning for dynamics and control}, pp.\  486--489. PMLR,
  2020.

\bibitem[Fang et~al.(2021)Fang, He, Long, and Su]{fang2021exploring}
Cong Fang, Hangfeng He, Qi~Long, and Weijie~J Su.
\newblock Exploring deep neural networks via layer-peeled model: Minority
  collapse in imbalanced training.
\newblock \emph{Proceedings of the National Academy of Sciences}, 118\penalty0
  (43):\penalty0 e2103091118, 2021.

\bibitem[Flet-Berliac et~al.(2024)Flet-Berliac, Grinsztajn, Strub, Choi, Wu,
  Cremer, Ahmadian, Chandak, Gheshlaghi~Azar, Pietquin, and Geist]{flet2024}
Yannis Flet-Berliac, Nathan Grinsztajn, Florian Strub, Eugene Choi, Bill Wu,
  Chris Cremer, Arash Ahmadian, Yash Chandak, Mohammad Gheshlaghi~Azar, Olivier
  Pietquin, and Matthieu Geist.
\newblock Contrastive policy gradient: Aligning {LLM}s on sequence-level scores
  in a supervised-friendly fashion.
\newblock In \emph{Proceedings of the 2024 Conference on Empirical Methods in
  Natural Language Processing}, pp.\  21353--21370, Miami, Florida, USA, 2024.
  Association for Computational Linguistics.
\newblock \doi{10.18653/v1/2024.emnlp-main.1190}.

\bibitem[Gardner \& Dorling(1998)Gardner and Dorling]{gardner1998artificial}
Matt~W Gardner and SR~Dorling.
\newblock Artificial neural networks (the multilayer perceptron)—a review of
  applications in the atmospheric sciences.
\newblock \emph{Atmospheric environment}, 32\penalty0 (14-15):\penalty0
  2627--2636, 1998.

\bibitem[Gaur et~al.(2023)Gaur, Bedi, Wang, and Aggarwal]{gaur2023global}
Mudit Gaur, Amrit~Singh Bedi, Di~Wang, and Vaneet Aggarwal.
\newblock On the global convergence of natural actor-critic with two-layer
  neural network parametrization.
\newblock \emph{arXiv preprint arXiv:2306.10486}, 2023.

\bibitem[Ghosh et~al.(2021)Ghosh, Gupta, and Levine]{ghosh2021learning}
Dibya Ghosh, Abhishek Gupta, and Sergey Levine.
\newblock Learning to reach goals via iterated supervised learning.
\newblock In \emph{International Conference on Learning Representations}, 2021.

\bibitem[Han et~al.(2021)Han, Papyan, and Donoho]{han2021neural}
XY~Han, Vardan Papyan, and David~L Donoho.
\newblock Neural collapse under mse loss: Proximity to and dynamics on the
  central path.
\newblock \emph{arXiv preprint arXiv:2106.02073}, 2021.

\bibitem[Husain et~al.(2021)Husain, Ciosek, and Tomioka]{regularized}
Hisham Husain, Kamil Ciosek, and Ryota Tomioka.
\newblock Regularized policies are reward robust.
\newblock In \emph{International Conference on Artificial Intelligence and
  Statistics}, pp.\  64--72. PMLR, 2021.

\bibitem[Jin et~al.(2018{\natexlab{a}})Jin, Allen-Zhu, Bubeck, and
  Jordan]{jin2018q}
Chi Jin, Zeyuan Allen-Zhu, Sebastien Bubeck, and Michael~I Jordan.
\newblock Is q-learning provably efficient?
\newblock \emph{Advances in neural information processing systems}, 31,
  2018{\natexlab{a}}.

\bibitem[Jin et~al.(2018{\natexlab{b}})Jin, Allen-Zhu, Bubeck, and
  Jordan]{qlearning_analysis}
Chi Jin, Zeyuan Allen-Zhu, Sebastien Bubeck, and Michael~I Jordan.
\newblock Is q-learning provably efficient?
\newblock \emph{Advances in neural information processing systems}, 31,
  2018{\natexlab{b}}.

\bibitem[Konda \& Tsitsiklis(1999)Konda and Tsitsiklis]{konda1999actor}
Vijay Konda and John Tsitsiklis.
\newblock Actor-critic algorithms.
\newblock \emph{Advances in neural information processing systems}, 12, 1999.

\bibitem[Li et~al.(2019)Li, Wang, Ma, Ortal, Zhao, Stenger, and
  Hirate]{li2019learning}
Tianyu Li, Chien-Chih Wang, Yukun Ma, Patricia Ortal, Qifang Zhao, Bjorn
  Stenger, and Yu~Hirate.
\newblock Learning classifiers on positive and unlabeled data with policy
  gradient.
\newblock In \emph{2019 IEEE International Conference on Data Mining (ICDM)},
  pp.\  399--408. IEEE, 2019.

\bibitem[Liang \& Davis(2023)Liang and Davis]{liang2023inducing}
Tong Liang and Jim Davis.
\newblock Inducing neural collapse to a fixed hierarchy-aware frame for
  reducing mistake severity.
\newblock \emph{arXiv preprint arXiv:2303.05689}, 2023.

\bibitem[Lillicrap et~al.(2015)Lillicrap, Hunt, Pritzel, Heess, Erez, Tassa,
  Silver, and Wierstra]{lillicrap2015continuous}
Timothy~P Lillicrap, Jonathan~J Hunt, Alexander Pritzel, Nicolas Heess, Tom
  Erez, Yuval Tassa, David Silver, and Daan Wierstra.
\newblock Continuous control with deep reinforcement learning.
\newblock \emph{arXiv preprint arXiv:1509.02971}, 2015.

\bibitem[Lin(1992)]{lin1992self}
Long-Ji Lin.
\newblock Self-improving reactive agents based on reinforcement learning,
  planning and teaching.
\newblock \emph{Machine learning}, 8:\penalty0 293--321, 1992.

\bibitem[Liu et~al.(2023)Liu, Zhang, Hu, Cao, Yao, and Pan]{liu2023inducing}
Xuantong Liu, Jianfeng Zhang, Tianyang Hu, He~Cao, Yuan Yao, and Lujia Pan.
\newblock Inducing neural collapse in deep long-tailed learning.
\newblock In \emph{International Conference on Artificial Intelligence and
  Statistics}, pp.\  11534--11544. PMLR, 2023.

\bibitem[Lu \& Steinerberger(2020)Lu and Steinerberger]{lu2020neural}
Jianfeng Lu and Stefan Steinerberger.
\newblock Neural collapse with cross-entropy loss.
\newblock \emph{arXiv preprint arXiv:2012.08465}, 2020.

\bibitem[Mixon et~al.(2020)Mixon, Parshall, and Pi]{mixon2020neural}
Dustin~G Mixon, Hans Parshall, and Jianzong Pi.
\newblock Neural collapse with unconstrained features.
\newblock \emph{arXiv preprint arXiv:2011.11619}, 2020.

\bibitem[Mnih et~al.(2016)Mnih, Badia, Mirza, Graves, Lillicrap, Harley,
  Silver, and Kavukcuoglu]{mnih2016asynchronous}
Volodymyr Mnih, Adria~Puigdomenech Badia, Mehdi Mirza, Alex Graves, Timothy
  Lillicrap, Tim Harley, David Silver, and Koray Kavukcuoglu.
\newblock Asynchronous methods for deep reinforcement learning.
\newblock In \emph{International conference on machine learning}, pp.\
  1928--1937. PMLR, 2016.

\bibitem[Papyan et~al.(2020)Papyan, Han, and Donoho]{papyan2020prevalence}
Vardan Papyan, XY~Han, and David~L Donoho.
\newblock Prevalence of neural collapse during the terminal phase of deep
  learning training.
\newblock \emph{Proceedings of the National Academy of Sciences}, 117\penalty0
  (40):\penalty0 24652--24663, 2020.

\bibitem[Peng et~al.(2023)Peng, Li, He, Galley, and Gao]{peng2023instruction}
Baolin Peng, Chunyuan Li, Pengcheng He, Michel Galley, and Jianfeng Gao.
\newblock Instruction tuning with gpt-4.
\newblock \emph{arXiv preprint arXiv:2304.03277}, 2023.

\bibitem[Rafailov et~al.(2023)Rafailov, Sharma, Mitchell, Manning, Ermon, and
  Finn]{rafailov2023direct}
Rafael Rafailov, Archit Sharma, Eric Mitchell, Christopher~D Manning, Stefano
  Ermon, and Chelsea Finn.
\newblock Direct preference optimization: Your language model is secretly a
  reward model.
\newblock \emph{Advances in Neural Information Processing Systems},
  36:\penalty0 53728--53741, 2023.

\bibitem[Schulman et~al.(2015)Schulman, Levine, Abbeel, Jordan, and
  Moritz]{schulman2015trust}
John Schulman, Sergey Levine, Pieter Abbeel, Michael Jordan, and Philipp
  Moritz.
\newblock Trust region policy optimization.
\newblock In \emph{International conference on machine learning}, pp.\
  1889--1897. PMLR, 2015.

\bibitem[Schulman et~al.(2017)Schulman, Wolski, Dhariwal, Radford, and
  Klimov]{schulman2017proximal}
John Schulman, Filip Wolski, Prafulla Dhariwal, Alec Radford, and Oleg Klimov.
\newblock Proximal policy optimization algorithms.
\newblock \emph{arXiv preprint arXiv:1707.06347}, 2017.

\bibitem[Silver et~al.(2014)Silver, Lever, Heess, Degris, Wierstra, and
  Riedmiller]{silver2014deterministic}
David Silver, Guy Lever, Nicolas Heess, Thomas Degris, Daan Wierstra, and
  Martin Riedmiller.
\newblock Deterministic policy gradient algorithms.
\newblock In \emph{International conference on machine learning}, pp.\
  387--395. Pmlr, 2014.

\bibitem[Simchowitz \& Jamieson(2019)Simchowitz and
  Jamieson]{regretbound_analysis}
Max Simchowitz and Kevin~G Jamieson.
\newblock Non-asymptotic gap-dependent regret bounds for tabular mdps.
\newblock \emph{Advances in Neural Information Processing Systems}, 32, 2019.

\bibitem[Sriyash et~al.(2024)Sriyash, Yanming, Hamish, Abhishek, and
  Natasha]{sriyash2024}
Poddar Sriyash, Wan Yanming, Ivison Hamish, Gupta Abhishek, and Jaques Natasha.
\newblock Personalizing reinforcement learning from human feedback with
  multimodal reward models.
\newblock In \emph{Advances in Neural Information Processing Systems
  (NeurIPS)}, 2024.

\bibitem[Strehl \& Littman(2005)Strehl and Littman]{strehl2005theoretical}
Alexander~L Strehl and Michael~L Littman.
\newblock A theoretical analysis of model-based interval estimation.
\newblock In \emph{Proceedings of the 22nd international conference on Machine
  learning}, pp.\  856--863, 2005.

\bibitem[Strehl et~al.(2006)Strehl, Li, Wiewiora, Langford, and
  Littman]{strehl2006pac}
Alexander~L Strehl, Lihong Li, Eric Wiewiora, John Langford, and Michael~L
  Littman.
\newblock Pac model-free reinforcement learning.
\newblock In \emph{Proceedings of the 23rd international conference on Machine
  learning}, pp.\  881--888, 2006.

\bibitem[Strohmer \& Heath~Jr(2003)Strohmer and Heath~Jr]{ETF}
Thomas Strohmer and Robert~W Heath~Jr.
\newblock Grassmannian frames with applications to coding and communication.
\newblock \emph{Applied and computational harmonic analysis}, 14\penalty0
  (3):\penalty0 257--275, 2003.

\bibitem[Sutton(1984)]{sutton1984on}
Richard~S. Sutton.
\newblock On-line q-learning using state-action-reward-state-action.
\newblock \emph{Machine Learning}, 9\penalty0 (4):\penalty0 323--332, 1984.

\bibitem[Sutton et~al.(1999)Sutton, McAllester, Singh, and
  Mansour]{sutton1999policy}
Richard~S Sutton, David McAllester, Satinder Singh, and Yishay Mansour.
\newblock Policy gradient methods for reinforcement learning with function
  approximation.
\newblock \emph{Advances in neural information processing systems}, 12, 1999.

\bibitem[Thrun(1992)]{thrun1992efficient}
Sebastian Thrun.
\newblock Efficient exploration in reinforcement learning.
\newblock \emph{Technical Report. Carnegie Mellon University}, 1992.

\bibitem[Watkins(1989)]{watkins1989learning}
Christopher~JCH Watkins.
\newblock Learning from delayed rewards.
\newblock \emph{Machine Learning}, 8\penalty0 (3-4):\penalty0 279--292, 1989.

\bibitem[Wu \& Papyan(2024)Wu and Papyan]{wu2024linguistic}
Robert Wu and Vardan Papyan.
\newblock Linguistic collapse: Neural collapse in (large) language models.
\newblock In \emph{Advances in Neural Information Processing Systems}, 2024.
\newblock URL \url{https://arxiv.org/abs/2405.17767}.

\bibitem[Xie et~al.(2023)Xie, Yang, Cai, and He]{xie2023neural}
Liang Xie, Yibo Yang, Deng Cai, and Xiaofei He.
\newblock Neural collapse inspired attraction–repulsion-balanced loss for
  imbalanced learning.
\newblock \emph{Neurocomputing}, 527:\penalty0 1--14, 2023.
\newblock \doi{10.1016/j.neucom.2023.01.023}.

\bibitem[Xiong et~al.(2022{\natexlab{a}})Xiong, Xu, Zhao, Liang, and
  Zhang]{deterministic}
Huaqing Xiong, Tengyu Xu, Lin Zhao, Yingbin Liang, and Wei Zhang.
\newblock Deterministic policy gradient: Convergence analysis.
\newblock In \emph{Uncertainty in Artificial Intelligence}, pp.\  2159--2169.
  PMLR, 2022{\natexlab{a}}.

\bibitem[Xiong et~al.(2022{\natexlab{b}})Xiong, Xu, Zhao, Liang, and
  Zhang]{xiong2022deterministic}
Huaqing Xiong, Tengyu Xu, Lin Zhao, Yingbin Liang, and Wei Zhang.
\newblock Deterministic policy gradient: Convergence analysis.
\newblock In \emph{Uncertainty in Artificial Intelligence}, pp.\  2159--2169.
  PMLR, 2022{\natexlab{b}}.

\bibitem[Yang et~al.(2022)Yang, Chen, Li, Xie, Lin, and Tao]{yang2022inducing}
Yibo Yang, Shixiang Chen, Xiangtai Li, Liang Xie, Zhouchen Lin, and Dacheng
  Tao.
\newblock Inducing neural collapse in imbalanced learning: Do we really need a
  learnable classifier at the end of deep neural network?
\newblock \emph{Advances in Neural Information Processing Systems},
  35:\penalty0 37991--38002, 2022.

\bibitem[Yang et~al.(2023)Yang, Yuan, Li, Lin, Torr, and Tao]{yang2023neural}
Yibo Yang, Haobo Yuan, Xiangtai Li, Zhouchen Lin, Philip Torr, and Dacheng Tao.
\newblock Neural collapse inspired feature-classifier alignment for few-shot
  class-incremental learning.
\newblock In \emph{The Eleventh International Conference on Learning
  Representations}, 2023.

\bibitem[Zhang et~al.(2019)Zhang, Koppel, Zhu, and
  Ba{\c{s}}ar]{zhang2019convergence}
Kaiqing Zhang, Alec Koppel, Hao Zhu, and Tamer Ba{\c{s}}ar.
\newblock Convergence and iteration complexity of policy gradient method for
  infinite-horizon reinforcement learning.
\newblock In \emph{2019 IEEE 58th Conference on Decision and Control (CDC)},
  pp.\  7415--7422. IEEE, 2019.

\bibitem[Zhang et~al.(2020)Zhang, Koppel, Zhu, and Basar]{zhang2020global}
Kaiqing Zhang, Alec Koppel, Hao Zhu, and Tamer Basar.
\newblock Global convergence of policy gradient methods to (almost) locally
  optimal policies.
\newblock \emph{SIAM Journal on Control and Optimization}, 58\penalty0
  (6):\penalty0 3586--3612, 2020.

\end{thebibliography}
